\documentclass[prodmode,acmtkdd]{acmsmall} 

\usepackage[ruled]{algorithm2e}
\renewcommand{\algorithmcfname}{ALGORITHM}
\SetAlFnt{\small}
\SetAlCapFnt{\small}
\SetAlCapNameFnt{\small}
\SetAlCapHSkip{0pt}
\IncMargin{-\parindent}

\acmVolume{}
\acmNumber{}
\acmArticle{}
\acmYear{2013}
\acmMonth{12}
\usepackage{amssymb,amsfonts,amsmath,bm}

\newlength{\defbaselineskip}
\usepackage{multirow}
\usepackage{graphicx,epstopdf}
\usepackage{enumerate,url}

\def\math#1{$#1$}

\def\mand#1{$$#1$$}
\def\v#1{{\mathbf #1}}
\def\frac#1#2{{#1\over #2}}

\def\mld#1{\begin{equation}
#1
\end{equation}}

\def\x{{\mathbf x}}
\def\y{{\mathbf y}}

\def\X{{\mathbf X}}

\def\a{{\mathbf a}}
\def\b{{\mathbf b}}

\def\norm#1{{\|#1\|}}

\def\r#1{{(\ref{#1})}}

\def\dotfil{\leaders\hbox to 1.5mm{.}\hfill}
\newcounter{rmnum}
\def\RN#1{\setcounter{rmnum}{#1}\uppercase\expandafter{\romannumeral\value{rmnum}}}
\def\rn#1{\setcounter{rmnum}{#1}\expandafter{\romannumeral\value{rmnum}}}

\newcommand{\TNorm }[1]{\mbox{}\left\|#1\right\|_2  }
\newcommand{\TNormS}[1]{\mbox{}\left\|#1\right\|_2^2}

\newcommand{\setlinespacing}[1]%
           {\setlength{\baselineskip}{#1 \defbaselineskip}}

\newcommand{\abs }[1]{\left|#1\right|}

\newcommand{\mat}[1]{{\ensuremath{\bm{\mathrm{#1}}}}}

\def\a{{\bm \alpha}}
\def\w{{\mathbf w}}
\def\hw{{\mathbf{\hat{w}}}}
\def\ta{\tilde{\bm \alpha}}
\def\tw{\tilde{\mathbf w}}
\def\b{{\mathbf b}}
\def\e{{\mathbf e}}

\def\x{{\mathbf x}}

\def\v{{\mathbf v}}

\def\matA{\mat{A}}
\def\matB{\mat{B}}

\def\matD{\mat{D}}
\def\matE{\mat{E}}

\def\matH{\mat{H}}
\def\matI{\mat{I}}

\def\matP{\mat{P}}
\def\matQ{\mat{Q}}
\def\matR{\mat{R}}
\def\matRhad{\matR_{\textsc{\tiny srht}}}
\def\matS{\mat{S}}

\def\matU{\mat{U}}
\def\matV{\mat{V}}

\def\matX{\mat{X}}
\def\matY{\mat{Y}}
\def\matZ{\mat{Z}}
\def\matSig{\mat{\Sigma}}

\newcommand{\remove}[1]{}

\begin{document}

\markboth{Paul et al.}{Random Projections for Linear Support Vector Machines}

\title{Random Projections for Linear Support Vector Machines}

\author{
SAURABH PAUL
\affil{
Rensselaer Polytechnic Institute
} 
CHRISTOS BOUTSIDIS
\affil{ 
IBM T.J. Watson Research Center  
} 
MALIK MAGDON-ISMAIL 
\affil {
Rensselaer Polytechnic Institute 
}
PETROS DRINEAS 
\affil{ 
Rensselaer Polytechnic Institute 
}
}

\begin{abstract}
Let $\matX$ be a data matrix of rank $\rho$, whose rows represent $n$ points in $d$-dimensional space. The linear support vector machine constructs a hyperplane separator that maximizes the 1-norm soft margin. We develop a new oblivious dimension reduction technique which is precomputed and can be applied to any input matrix ${\matX}$. We prove that, with high probability, the margin and minimum enclosing ball in the feature space are preserved to within \math{\epsilon}-relative error, ensuring comparable generalization as in the original space in the case of classification. For regression, we show that the margin is preserved to $\epsilon$-relative error with high probability. We present extensive experiments with real and synthetic data to support our theory.
\end{abstract}




\acmformat{Saurabh Paul, Christos Boutsidis, Malik Magdon-Ismail, Petros Drineas, 2013. Random Projections for Linear Support Vector Machines.}

\begin{bottomstuff}
A short version of this paper appeared in the 16th International Conference on Artificial Intelligence and Statistics (AISTATS 2013) \cite{PBMD13}. Note, that the short version of our paper \cite{PBMD13} does not include the details of the proofs, comparison of random projections with principal component analysis, extension of random projections for SVM regression in terms of both theory and experiments and experiments with fast SVM solver on RCV1 and Hapmap-HGDP datasets.

Christos Boutsidis acknowledges the support from XDATA program of the Defense Advanced Research Projects Agency (DARPA), administered through Air Force Research Laboratory contract FA8750-12-C-0323; Petros Drineas and Malik Magdon-Ismail are supported by NSF CCF-1016501 and NSF DMS-1008983; Saurabh Paul is supported by NSF CCF-916415. 

Author's addresses: S. Paul {and} M. Magdon-Ismail {and} P. Drineas, Computer Science Department,
Rensselaer Polytechnic Institute, pauls2@rpi.edu {and} \{magdon, drinep\}@cs.rpi.edu ; C. Boutsidis, Mathematical Sciences Department, IBM T.J. Watson Research Center, cboutsi@us.ibm.com.
\end{bottomstuff}

\maketitle

\section{Introduction}\label{sec:intro}

Support Vector Machines (SVM) \cite{Chris00}
are extremely popular in machine learning today. They have been used in both classification and regression.
For classification, the training
data set consists
of $n$ points $\x_i \in \mathbb{R}^d$,
with respective labels $y_i \in \{-1,+1\}$ for $i=1\ldots n$.
For linearly separable data, the primal form of the SVM learning problem
is to construct a hyperplane
\math{\w^*} which maximizes the geometric
\emph{margin}
(the minimum distance of a data point to the
hyperplane), while separating the data.
For non-separable data the ``soft'' 1-norm margin is maximized.
The dual lagrangian formulation of the classification problem leads to the
following quadratic program:
\begin{equation}
\begin{aligned}
& \max_{\{ \alpha_i \} } \sum_{i=1}^n \alpha_i  - \frac{1}{2} \sum_{i,j=1}^n \alpha_i \alpha_j y_i y_j \x_i^T\x_j \\
& \text{subject to} \sum_{i=1}^{n}y_i \alpha_i = 0, \\   
& 0 \leq  \alpha_i \leq C,  \; \; \; i=1\ldots n.
\end{aligned}
\label{eqn:svm1}
\end{equation}
In the above formulation, the unknown lagrange multipliers
 $\{\alpha_i\}_{i=1}^n$ are constrained to lie inside the ``box constraint'' \math{[0,C]^n},
 where $C$ is part of the input.
In order to measure the out-of-sample performance of the SVM classifier, we can use the VC-dimension of \emph{fat}-separators. Assuming that the data lie in a
ball of radius \math{B}, and that the hypothesis set consists of hyperplanes of width \math{\gamma} (corresponding to the margin), then the
\math{VC}-dimension of this hypothesis set is \math{O(B^2/\gamma^2)} \cite{Vapnik71}.
Now, given the in-sample error, we can obtain a bound for the out-of-sample error, which is monotonic in the VC-dimension \cite{Vapnik98}.

Analogous to the 1-norm soft margin formulation for SVM classification, we have a similar formulation for regression called the linear $\varepsilon$-insensitive loss SVM \cite{Chris00}. The dual problem for $\varepsilon$-insensitive loss SVM regression is formulated as:
%
\begin{equation}
\begin{aligned}
& \max \sum_{i=1}^n\alpha_iy_i - \varepsilon \sum_{i=1}^n \abs{\alpha_i} - \frac{1}{2} \sum_{i,j=1}^n \alpha_i\alpha_j \x_i^T\x_j \\
& \text{subject to} \sum_{i=1}^{n} \alpha_i  = 0, \\
& -C \leq  \alpha_i \leq C,  \; \; \; i=1\ldots n.
\end{aligned}
\label{eqn:svm4}
\end{equation}
\noindent Here, $\{\alpha\}_{i=1}^n$ are the Lagrange multipliers and they lie in the interval \math{[-C,C]^n.}

Intuitively, if one can preserve the subspace geometry, then one should be able to preserve the performance of a distance-based algorithm.
We construct dimension reduction matrices \math{\matR\in\mathbb{R}^{d \times r}} which  produce \math{r}-dimensional feature vectors \math{\tilde \x_i=\matR^T\x_i}; the matrices \math{\matR} \emph{do not} depend on the data.
We show that for the data in the dimension-reduced space, the margin of separability and the minimum enclosing ball radius are preserved, since the subspace geometry is preserved. So, an SVM with an
appropriate structure defined by the margin (width) of the hyperplanes \cite{Vapnik98}
will have comparable VC-dimension and, thus, generalization error. This is true for classification. The $\varepsilon$-insensitive loss SVM regression problem is an unbounded problem and as such, we are not able to infer anything related to the generalization error bounds: we can only infer the preservation of margin.

\subsection{Notation and SVM Basics}

$\matA, \matB, \ldots$ denote matrices and $\a, \b, \ldots$ denote column vectors; $\e_i$ (for all $i=1\ldots n$) is the standard basis, whose dimensionality will be clear from context; and $\matI_n$ is the  $n \times n$ identity matrix. The Singular Value Decomposition (SVD) of a matrix $\matA \in \mathbb{R}^{n \times d}$ of rank $\rho \leq \min\left\{n,d\right\}$ is equal to %
$ \matA = \matU \matSig \matV^T,$
where $\matU \in \mathbb{R}^{n \times \rho}$ is an orthogonal matrix containing the left singular vectors, $\matSig \in \mathbb{R}^{\rho \times \rho}$ is a diagonal matrix containing the singular values $\sigma_1 \geq \sigma_2  \geq \ldots \sigma_{\rho} > 0$, and $\matV \in \mathbb{R}^{d \times \rho}$ is a matrix containing the right singular vectors.
The spectral norm of $\matA$ is  $\TNorm{\matA} = \sigma_1$.
\label{subsec:svm_basics}
\subsubsection{SVM Classification}
Let $\matX \in \mathbb{R}^{n \times d}$ be the matrix whose rows are the vectors $\x_i^T$, $\matY \in \mathbb{R}^{n \times n}$ be the  diagonal matrix with entries  $\matY_{ii}=y_{i}$, and  $\a = \left[\alpha_1,\alpha_2,\ldots,\alpha_n\right] \in \mathbb{R}^n$ be
the vector of lagrange multipliers to be determined by solving eqn.~\r{eqn:svm1}. The SVM optimization
problem is
\begin{equation}
\begin{aligned}
& \max_{\a } \bm{1}^T\a-\frac12\a^T\matY\matX\matX^T\matY\a \\
& \text{subject to }  \bm 1^T \matY \a = 0;  \qquad \mbox{and} \qquad
\bm{0} \le\a \le \bm{C}.
\end{aligned}
\label{eqn:svm1A}
\end{equation}
(In the above, \math{\bm1,\ \bm0,\ \bm{C}} are  vectors with the implied constant entry.)
Let \math{\a^*} be an optimal solution of the above problem. The optimal separating hyperplane is given by $\w^* = \matX^T\matY\a^{*}= \sum_{i=1}^n y_i\alpha_i^* \x_i $, and the points $\x_i$ for which \math{\alpha_i^*>0}, i.e., the points which appear in the expansion \math{\w^*}, are the support vectors.
The geometric margin, \math{\gamma^*}, of this
canonical optimal hyperplane is \math{\gamma^*=1/\TNorm{\w^*}}, where
$\TNormS{\w^*} = \sum_{i=1}^n \alpha_i^*$.
The data radius is \math{B=\min_{\x^*}\max_{\x_i} \TNorm{\x_i-\x^*}}.
It is this \math{\gamma^*} and \math{B} that factor into the generalization
performance of the SVM through the ratio \math{B/\gamma^*}. It is worth noting, that our results hold for the separable case as well, which amounts to setting $C$ to a large value.
\subsubsection{SVM Regression}
Let $\matX \in \mathbb{R}^{n \times d}$ be the matrix whose rows are the vectors $\x_i^T$, $\y$ be the n-dimensional vector with the target entries, and  $\a = \left[\alpha_1,\alpha_2,\ldots,\alpha_n\right] \in \mathbb{R}^n$ be
the vector of lagrange multipliers to be determined by solving eqn.~\r{eqn:svm4}. The SVM optimization
problem is
\begin{equation}
\begin{aligned}
& \max_{\a} \y^T\a - \varepsilon\bm{1}^T\a - \frac12\a^T\matX\matX^T\a \\
& \text{subject to }  \bm 1^T \a = 0;  \qquad \mbox{and} \qquad
\bm{-C} \le\a \le \bm{C}.
\end{aligned}
\label{eqn:svm4A}
\end{equation}

\noindent \math{\bm1,\ \bm0,\ \bm{C}} are vectors with the implied constant entry. Let \math{\a^*} be an optimal solution of the above problem. The optimal separating hyperplane is given by $\w^* = \a^{*T}\matX= \sum_{i=1}^n\alpha_i^* \x_i $ and the points $\x_i$ for which \math{\alpha_i^*>0}, i.e., the points which appear in the expansion \math{\w^*}, are the support vectors. The geometric margin for regression is defined in the same way as it was done for classification.

\subsection{Dimension Reduction}
\label{subsec:dim_red}
Our goal is to study how the SVM performs under (linear) dimensionality reduction transformations in the feature space. Let $\matR \in \mathbb{R}^{d \times r}$ be the dimension reduction matrix that reduces the dimensionality of the input from \math{d} to  $r \ll d$. We will choose \math{\matR} to be a random projection matrix (see Section~\ref{subsec:prel}). The transformed dataset into \math{r} dimensions is given by \math{\tilde\matX=\matX\matR}, and the SVM optimization problem for classification becomes
\begin{equation}
\begin{aligned}
& \max_{\tilde\a } \bm{1}^T\tilde\a-\frac12\tilde\a^T\matY\matX\matR\matR^T
\matX^T\matY\tilde\a, \\
& \text{subject to } \bm1^T\matY\tilde\a=0, \qquad \mbox{and} \qquad
&  \bm0\le\tilde\a\le \bm{C}.
\end{aligned}
\label{eqn:svm2}
\end{equation}
For regression, the SVM optimization problem becomes
\begin{equation}
\begin{aligned}
& \max_{\tilde\a } \y^T\tilde\a - \varepsilon\bm{1}^T\tilde\a - \frac12\tilde\a^T\matX\matR\matR^T\matX^T\tilde\a \\
& \text{subject to }  \bm 1^T \tilde\a = 0;  \qquad \mbox{and} \qquad
\bm{-C} \le\tilde\a \le \bm{C}.
\end{aligned}
\label{eqn:svm5}
\end{equation}
\noindent
We will present a construction for $\matR$ that leverages the fast Hadamard transform. The running time needed to apply this construction to the original data matrix is $O\left( nd\log r\right)$. Notice that while this running time is nearly linear on the size of the original data, it does not take advantage of any sparsity in the input. In order to address this deficiency, we leverage the recent work of \citeN{Clark12}, \citeN{Meng13} and \citeN{NN13}, which proposes a construction for $\matR$ that can be applied to $\X$ in $O\left(nnz(\X) + poly\left(n\epsilon^{-1} \right) \right)$ time; here $nnz\left(\X \right)$ denotes the number of non-zero entries of $\X$ and $\rho$ is the rank of $\X$. To the best of our knowledge, this is the first independent implementation and evaluation of this potentially ground-breaking random projection technique (a few experimental results were presented by \citeN{Clark12}, \citeN{Meng13} and \citeN{NN13}). All constructions for $\matR$ are oblivious of the data and hence they can be precomputed. Also, the generalization bounds that depend on the final margin and radius of the data will continue to hold for classification, while the bound on the margin holds for regression.

The pratical intent of using linear SVM after random projections is to reduce computational complexity of training SVM and memory. For large-scale datasets, that are too big to fit into memory (see Section ~\ref{subsubsec:snp} for details), random projections serve as a possible way to estimate out-of-sample error. Random projections reduce the computational complexity of training SVM which is evident from the experiments described in Section ~\ref{sec:exp}.  

\subsection{Our Contribution}
\label{subsec:our_contri}
Our main theoretical results are to show that by solving the SVM optimization problem in the projected space, we get relative-error generalization performance for SVM classification and that, we preserve the margin upto relative error for SVM regression.
We briefly discuss the appropriate values of $r$, namely the dimensionality of the dimensionally-reduced problem. If $\matR$ is the matrix of the randomized Hadamard transform (see Section~\ref{subsec:prel} for details), then given $\epsilon\in(0,1/2]$ and $\delta\in\left(0,1\right)$ we set
\begin{equation}\label{eqn:eqfht}
r=O\left(\rho\epsilon^{-2}\cdot\log\left(\rho d\delta^{-1}\right)\cdot\log\left(\rho\epsilon^{-2}\delta^{-1}\log\left(\rho d\delta^{-1}\right)\right)\right).
\end{equation}
The running time needed to apply the randomized Hadamard transform is \math{O\left(nd\log r\right)}.
If $\matR$ is constructed as described in~\citeN{Clark12},~\citeN{Meng13} and \citeN{NN13}, then given  $\epsilon\in\left(0,1\right)$ and $\delta\in\left(0,1\right)$ we set
\begin{equation} \label{eqn:eqcw}
r=O\left(\rho\epsilon^{-4}\log\left(\rho/\delta\epsilon\right) \left(\rho+\log\left(1/\delta\epsilon\right)\right)\right).
\end{equation}
The running time needed to apply this transform is \math{O\left(nnz(\X)+poly\left(n\epsilon^{-1}\right)\right)}. If $\matR$ is a random sign-matrix, then given $\epsilon\in(0,1/2]$ and we set
\begin{equation}\label{eqn:eqrs}
r=O\left(\rho\epsilon^{-2}\log\rho\log d\right).
\end{equation}
The running time needed to apply this transform is equal to $O\left(ndr\right)$. Finally if $\matR$ is a random gaussian matrix, then given $\epsilon\in(0,1/2]$ and $\delta\in\left(0,1\right)$ we set,
\begin{equation}\label{eqn:eqrg}
r=O\left(\rho\epsilon^{-2}\log\left( \rho /\delta\right)\right).
\end{equation}

Our main theorem will be stated in terms of the randomized Hadamard Transform, but similar statements can be obtained for the other three transforms.
\begin{theorem}
\label{thm:genthm}
Let $\epsilon\in(0,1/2]$ be an accuracy paramater and let $\delta\in\left(0,1\right)$ be a failure probability. Let \math{\matR\in\mathbb{R}^{d\times r}} be the matrix of the randomized Hadamard Transform, with $r$ as in eqn.~(\ref{eqn:eqfht}). Let \math{\gamma^*} and \math{\tilde\gamma^*} be the margins obtained by solving the SVM problems using data matrices \math{\matX} and $\matX \matR$ respectively (eqns.~(\ref{eqn:svm1A}) and~(\ref{eqn:svm2})). Let $B$ be the radius of the minimum ball enclosing all points in the full-dimensional space (rows of $\matX$) and let $\tilde{B}$ be the radius of the ball enclosing all points in the dimensionally-reduced space (rows of $\matX \matR$). Then, with probability at least $1-2\delta$,
$$\frac{\tilde{B}^2}{\tilde{\gamma}^{*2}} \leq \frac{\left(1+\epsilon\right)}{\left(1-\epsilon\right)}\frac{B^2}{\gamma^{*2}}.$$ 	
\end{theorem}
Similar theorems can be stated for the other two constructions of $\matR$ by setting the value of $r$ as in eqns.~(\ref{eqn:eqcw}), ~(\ref{eqn:eqrs}) and ~(\ref{eqn:eqrg}). For the case of SVM regression, we can only show that the margin is preserved up to relative error with probability at least $1-\delta$, namely
$$\tilde\gamma^{*2}\ge \left(1-\epsilon \right)\gamma^{*2}.$$

\subsection{Prior work}
\label{subsec:prior}
The work most closely related to our results is that of \citeN{Krish07}, which improved upon \citeN{Balca01}. \citeN{Balca01} and \citeN{Balca02} used random sampling techniques for solving the SVM classification and $\varepsilon$-insensitive loss SVM regression problem respectively, but they were not able to implement their algorithms in practice. \citeN{Krish07} and \citeN{Krish09} showed that by using sub-problems based on Gaussian random projections, one can obtain a solution to the SVM classification and regression problem with a margin that is relative-error close to the optimal. Their sampling complexity (the parameter $r$ in our parlance) depends on \math{B^4}, and, most importantly, on \math{1/{\gamma^*}^2}. This bound is not directly comparable to our result, which only depends on the rank of the data manifold, and holds regardless of the margin of the original problem (which could be arbitrarily small). Our results dramatically improve the running time needed to apply the random projections; our running times are (theoretically) linear in the number of non-zero entries in $\X$, whereas \citeN{Krish07} necessitates $O(ndr)$ time to apply $\matR$ on $\X$.

\citeN{Blum05} showed relative error margin preservation for linearly separable data by angle preservation between points when using random orthogonal matrix, standard gaussian matrix and the random sign matrix. We show relative-error margin preservation for non-separable data and use methods that improve running time to compute random projections. \citeN{Shi12}  establish the conditions under which margins are preserved after random projection and show that error free margins are preserved for both binary and multi-class problems if these conditions are met. They discuss the theory of margin and angle preservation after random projections using Gaussian matrices. They show that margin preservation is closely related to acute angle preservation and inner product preservation. Smaller acute angle leads to better preservation of the angle and the inner product. When the angle is well preserved, the margin is well-preserved too. There are two main differences between their result and ours. They show margin preservation to within additive error, whereas we give margin preservation to within relative error. This is a big difference especially when the margin is small. Moreover, they analyze only the separable case. We analyze the general non-separable dual problem and give a result in terms of the norm of the weight vector. For the separable case, the norm of the weight vector directly relates to the margin. For the non-separable case, one has to analyze the actual quadratic program, and our result essentially claims that the solution in the transformed space will have comparably regularized weights as the solution in the original space.

\citeN{Shi09} used hash kernels which approximately preserved inner product to design a biased approximation of the kernel matrix. The hash kernels can be computed in the number of non-zero terms of a data matrix similar to the method of \citeN{Clark12}, \citeN{Meng13} and \citeN{NN13} that we employed. \citeN{Shi09} used random sign matrices to compute random projections which typically increase the number of non-zero terms of the data matrix. However, the method of \citeN{Clark12}, \citeN{Meng13} and \citeN{NN13} takes advantage of input sparsity. \citeN{Shi09} showed that their generalization bounds on the hash kernel and the original kernel differed by the inverse of the product of the margin and number of datapoints. For smaller margins, this difference will be high. Our generalization bounds are independent of the original margin and hold for arbitrarily small margins.

\citeN{Zhang12} developed algorithms to accurately recover the optimal solution to the original SVM optimization problem using a Gaussian random projection. They compute the dual solution provided that the data matrix has low rank. This is different from our work since we analyze the ratio of radius of the minimum enclosing ball to the margin using random projections and do not try to recover the solution.

Finally, it is worth noting that random projection techniques have been applied extensively in the compressed sensing literature, and our theorems have the same flavor to a number of results in that area. However, to the best of our knowledge, the compressed sensing literature has not investigated the 1-norm soft-margin SVM optimization problem. 
\vskip -0.05 in
\section{Random Projection Matrices}
\label{subsec:prel}
Random projections are extremely popular techniques in order to deal with the curse-of-dimensionality. Let the data matrix be $\matX \in \mathbb{R}^{n \times d}$ ($n$ data points in $\mathbb{R}^d$) and let $\matR \in \mathbb{R}^{d \times r}$ (with $r \ll d$) be a random projection matrix. Then, the projected data matrix is $\tilde{\matX}=\matX\matR\in\mathbb{R}^{n\times  r}$ (\math{n} points in \math{\mathbb{R}^r}). If $\matR$ is carefully chosen, then all pairwise Euclidean distances are preserved with high probability. Thus, the geometry of the set of points in preserved, and it is reasonable to hope that an optimization objective such as the one that appears in SVMs will be only mildly perturbed.

There are many possible constructions for the matrix $\matR$ that preserve pairwise distances. The most common one is a matrix $\matR$ whose entries are i.i.d. standard Gaussian random variables \cite{Dasgu03,Indyk98} --\textbf{RG} for short. \citeN{Achli03} argued that the random sign matrix -- \textbf{RS} for short -- e.g., a matrix whose entries are set to $+1$ or $-1$ with equal probability,
also works. \citeN{Li06} used the sparse random projection matrix whose entries were set to $+1$ or $-1$ with probability $1/2\sqrt{d}$ and $0$ with probability $(1- 1/\sqrt{d})$. These constructions take \math{O\left(ndr\right)} time to compute \math{\tilde\matX}. 

More recently, faster methods of constructing random projections have been developed, using, for example, the Fast Hadamard Transform \cite{Ailon06} -- \textbf{FHT} for short.  The Hadamard-Walsh matrix for any $d$ that is a power of two is defined as
$$ \matH_d = \left[
\begin{array}{cc}
  \matH_{d/2} & \matH_{d/2} \\
  \matH_{d/2} & -\matH_{d/2}
\end{array}\right] \in \mathbb{R}^{d \times d},$$ \\
\text{with} $\matH_1 = +1.$
The normalized Hadamard-Walsh matrix is $\sqrt{\frac{1}{d}}\matH_d$,
which we simply
denote by $\matH$.
We set:
\mld{\matRhad=\sqrt{\frac{d}{r}} \matD\matH\matS,\label{eqn:had}}
a rescaled product of three matrices.
$\matD \in \mathbb{R}^{d \times d}$ is a random diagonal matrix
with $\matD_{ii}$ equal to $\pm1$ with probability \math{\frac12}.
 $\matH \in \mathbb{R}^{d\times d}$ is
the normalized Hadamard transform matrix.
\math{\matS\in\mathbb{R}^{d\times r}} is a random \emph{sampling matrix} which randomly samples
columns of \math{\matD\matH}; specifically, each of the \math{r}
columns of
\math{\matS} is independent and selected uniformly at random
(with replacement)
from the columns of \math{\matI_d}, the identity matrix. This
construction assumes that \math{d} is a power of two. If not, we just pad
\math{\X} with columns of zeros (affecting run times by at most a factor of
two). The important property of this transform is that the projected
features \math{\tilde\matX=\matX\matR} can be computed efficiently in
$O\left(nd\log r\right)$ time (see Theorem 2.1 of \citeN{Ailon2008} for details).
An important property of \math{\matR} (that follows from prior work) is
that it preserves orthogonality. 

While the randomized Hadamard transform is a major improvement over prior work, it does not take advantage of any sparsity in the input matrix. To fix this, very recent work \cite{Clark12} shows that carefully constructed random projection matrices can be applied in input sparsity time by making use of generalized sparse embedding matrices. \citeN{Meng13}, \citeN{NN13} also use a similar construction which runs in input sparsity time. Here we describe the construction of \citeN{Clark12}. To understand their construction of $\matR$, assume that the rank of $\matX$ is $\rho$ and let $r = O\left(\rho \epsilon^{-4}\log \left(\rho/\delta\epsilon\right) \left(\rho+\log \left(1/\epsilon\delta\right)\right)\right).$ 
Then, let $a=\Theta\left(\epsilon^{-1} \log\left(\rho/\epsilon\delta\right)\right)$,  $v=\Theta\left(\epsilon^{-1}\right)$, and let $q= O\left(\rho \epsilon^{-2}\left(\rho+\log \left(1/\epsilon\delta\right)\right)\right)$ be an integer (by appropriately choosing the constants). 
The construction starts by letting $h:1\ldots d \rightarrow 1\ldots q$ be a random hash function; then, for $i=1\ldots q$, let $a_i = \abs{ h^{-1}(i)} $ and let $d = \sum_{i=1}^q a_i$. 
The construction proceeds by creating $q$ independent matrices $\matB_1\ldots \matB_q$, such that $\matB_i \in \mathbb{R}^{va \times a_i}$. Each $\matB_i$ is the concatenation (stacking the rows of matrices on top of each other) of the following matrices: $\sqrt{\frac{1}{a}}\mathbf{\Phi_{1}}\matD_1 \ldots \sqrt{\frac{1}{a}}\mathbf{\Phi}_{a} \matD_{a}$. 
The matrix $\mathbf{\Phi_{i}}\matD_{i} \in \mathbb{R}^{v \times a_i}$ is defined as follows: for each $m \in \{1 \ldots v \}$, $h(m) = g'$, where $g'$ is selected from $\{1 \ldots a_{i}\}$ uniformly at random. $\mathbf{\Phi_i} $ is a $v \times a_{i}$ binary matrix with $ \mathbf{\Phi_{h(m),m}} =1$ and all remaining entries set to zero. $\matD$ is an $ a_i \times a_i $ random diagonal matrix, with each diagonal entry independently set to be $+1$ or $-1$ with probability $1/2$. Finally, let $\matS$ be the block diagonal matrix constructed by stacking the $\matB_i$'s across its diagonal and let $\matP$ be a $d \times d$ permutation matrix; then,  $\matR=\left(\matS\matP\right)^T$. The running time is $O\left( nnz(\matX) + poly\left(n \epsilon^{-1} \right)\right) $. We will call the method of \citeN{Clark12} to construct a sparse embedding matrix \textbf{CW}.

\section{Geometry of SVM is preserved under Random Projection}
\label{sec:ourresults}

We now state and prove our main result, namely that solving the SVM optimization problem in the projected space results in comparable margin and data radius as in the original space. The following lemma will be crucial in our proof.
\begin{lemma}\label{lem:spectral}
Fix \math{\epsilon\in(0,\frac12]}, \math{\delta\in(0,1]}. Let \math{\matV\in\mathbb{R}^{d\times\rho}} be any matrix with orthonormal columns and let \math{\matR=\matRhad} as
in eqn.~(\ref{eqn:had}), with
\math{r=O(\rho\epsilon^{-2}\cdot \log(\rho d\delta^{-1})\cdot
\log(\rho\epsilon^{-2}\delta^{-1}\log(\rho d\delta^{-1})))}. Then,
with probability at least \math{1-\delta},
\mand{\norm{\matV^T\matV-\matV^T\matR\matR^T\matV}_2\le\epsilon.}
\end{lemma}

\begin{proof}
Consider the matrix $\matV^T \matR = \matV^T \matD\matH\matS$. Using Lemma 3 of~\cite{Drine11},
$$\TNormS{\left(\matH\matD\matV\right)_{(i)}} \leq \frac{2\rho \ln(40d\rho )}{d}
\Rightarrow \left(2\ln\left(40d\rho\right)\right)^{-1}\frac{\TNormS{\left(\matH\matD\matV\right)_{(i)}}}{\rho} \leq \frac{1}{d}$$

\noindent holds for all $i=1,\ldots,d$ with probability at least $1 -\delta$. In the above, the notation $\matA_{(i)}$ denotes the $i$-th row of $\matA$ as a row vector. Applying Theorem 4 with $\beta = \left(2\ln\left(40d\rho\right)\right)^{-1}$ (~\cite{Drine11}, Appendix) concludes the lemma.
\end{proof}

We now state two similar lemmas that cover two additional constructions for $\matR$.

\begin{lemma} \label{lem:spectral2}
Let  \math{\epsilon\in(0,\frac12]} and let \math{\matV\in\mathbb{R}^{d\times\rho}} be any matrix with orthonormal columns. 
Let $\matR \in \mathbb{R}^{d \times r}$ be a (rescaled) random sign matrix. If $r=O\left(\rho\epsilon^{-2} \log \rho \log d\right)$, then
with probability at least $1-1/n$, \mand{\norm{\matV^T\matV-\matV^T\matR\matR^T\matV}_2\le\epsilon.}
\end{lemma}

\begin{proof}
The proof of this result is a simple application of Theorem 3.1(i) of~\citeN{Zouzi11}.
\end{proof}

\begin{lemma} \label{lem:spectral3}
Let  $\epsilon\in(0,\frac12]$, $\delta\in\left(0,1\right)$, and let \math{\matV\in\mathbb{R}^{d\times\rho}} be any matrix with orthonormal columns. 
Let $\matR \in \mathbb{R}^{d \times r}$ be the CW random projection matrix (see Section~\ref{subsec:prel}) with $r =O\left(\rho\epsilon^{-4}\log\left(\rho /\delta\epsilon\right)\left(\rho+\log\left(1/\epsilon\delta\right)\right)\right)$. Then, with probability at least $1-\delta$, 
$$\norm{\matV^T\matV-\matV^T\matR\matR^T\matV}_2\le\epsilon.$$
\end{lemma}

\begin{proof}
The proof of this result follows from Theorem 1 of~\citeN{Meng13}.
\end{proof}

\begin{lemma} \label{lem:spectral4}
Let  $\epsilon\in(0,\frac12]$, $\delta\in\left(0,1\right)$, and let \math{\matV\in\mathbb{R}^{d\times\rho}} be any matrix with orthonormal columns. 
Let $\matR \in \mathbb{R}^{d \times r}$ be the Gaussian random projection matrix with $r=O\left(\rho\epsilon^{-2}\log\left( \rho /\delta\right)\right)$.
Then with probability at least $1-\delta$, 
$$\norm{\matV^T\matV-\matV^T\matR\matR^T\matV}_2\le\epsilon.$$
\end{lemma}

\begin{proof}
The proof of this result follows from Corollary 6 of ~\citeN{Zhang12}.
\end{proof}

Lemma ~\ref{lem:spectral4} does not have the log factors as in Lemma ~\ref{lem:spectral}, but Gaussian projections are slower since they require full matrix-matrix multiplications.

\begin{theorem}\label{thm:main_technical}
Let $\epsilon$ be an accuracy parameter and let \math{\matR\in\mathbb{R}^{d\times r}} be a matrix satisfying
\math{\norm{\matV^T\matV - \matV^T \matR\matR^T \matV}_2\le\epsilon}. Let \math{\gamma^*} and \math{\tilde\gamma^*} be the margins obtained by
solving the SVM problems using data matrices \math{\matX} and \math{\matX \matR}
respectively (eqns.~(\ref{eqn:svm1A}) and~(\ref{eqn:svm2})). Then,
$$\tilde\gamma^{*2}\ge \left(1-\epsilon\right)\cdot \gamma^{*2}.$$
\end{theorem}
\begin{proof} 
\noindent
Let
$\matE=\matV^T\matV - \matV^T \matR\matR^T \matV$, and
 $\a^* = \left[\alpha_1^*,\alpha_2^*,\ldots,\alpha_n^*\right]^T \in \mathbb{R}^n$ be the vector achieving the optimal solution for the problem of eqn.~(\ref{eqn:svm1A}) in Section~\ref{sec:intro}. Then,

\begin{eqnarray}
Z_{opt} &=& \sum_{i=1}^n \alpha_i^* - \frac{1}{2} \a^{*T} \matY \matX \matX^T \matY \a^* \nonumber \\
 &=& \sum_{i=1}^n \alpha_i^*  - \frac{1}{2} \a^{*T} \matY \matU \matSig \matV^T \matV \matSig \matU^T \matY \a^* \nonumber \\
\label{eqn: thm_eqn1} &=& \sum_{i=1}^n \alpha_i^* - \frac{1}{2} \a^{*T} \matY \matU \matSig \matV^T \matR \matR^T \matV \matSig \matU^T \matY \a^* \nonumber \\
& & - \frac{1}{2} \a^{*T} \matY \matU \matSig \matE \matSig \matU^T \matY \a^* .
\end{eqnarray}

Let $\ta^* = \left[\tilde{\alpha}_1^*,\tilde{\alpha}_2^*,\ldots,\tilde{\alpha}_n^*\right]^T \in \mathbb{R}^n$ be the vector achieving the optimal solution for the dimensionally-reduced SVM problem of eqn.~(\ref{eqn:svm2})
using \math{\tilde\matX=\matX\matR}. Using the SVD of $\matX$, we get

\begin{equation}
\label{eqn: thm_eqn2}
\tilde{Z}_{opt} = \sum_{i=1}^n \tilde{\alpha}_i^* - \frac{1}{2}\ta^{*T} \matY\matU \matSig \matV^T \matR \matR^T \matV \matSig \matU^T \matY \ta^*.
\end{equation}
Since the constraints on \math{\a^*,\ta^*} do not depend on the data
(see eqns.~(\ref{eqn:svm1A}) and~(\ref{eqn:svm2})), it is clear
that $\ta^*$ is a feasible solution for the problem of eqn.~(\ref{eqn:svm1A}). Thus, from the optimality of $\a^*$, and using eqn.~(\ref{eqn: thm_eqn2}), it follows that
\begin{eqnarray}
Z_{opt} &=& \sum_{i=1}^n \alpha_i^*  - \frac{1}{2} \a^{*T} \matY \matU \matSig \matV^T \matR \matR^T \matV \matSig \matU^T \matY \a^* \nonumber \\
& & - \frac{1}{2} \a^{*T} \matY \matU \matSig \matE \matSig \matU^T \matY \a^* \nonumber \\
&\geq& \sum_{i=1}^n \tilde{\alpha}_i^*  - \frac{1}{2} \ta^{*T} \matY \matU \matSig \matV^T \matR \matR^T \matV \matSig \matU^T \matY \ta^* \nonumber \\
& & - \frac{1}{2} \ta^{*T} \matY \matU \matSig \matE \matSig \matU^T \matY \ta^* \nonumber \\
\label{eqn: thm_eqn3} &=& \tilde{Z}_{opt} - \frac{1}{2} \ta^{*T} \matY \matU \matSig \matE \matSig \matU^T \matY \ta^*.
\end{eqnarray}
We now analyze the second term using standard
sub-multiplicativity properties and $\matV^T\matV = \matI$. Taking $\matQ = \ta^{*T}\matY\matU\matSig$
\begin{eqnarray}
\nonumber \frac{1}{2}\ta^{*T}\matY\matU\matSig\matE\matSig\matU^T\matY\ta^*
&\leq& \frac{1}{2}\TNorm{\matQ} \TNorm{E} \TNorm{\matQ^T} \nonumber \\
\nonumber &=& \frac{1}{2} \TNorm{\matE} \TNormS{\matQ} \\
\nonumber &=& \frac{1}{2} \TNorm{\matE} \TNormS{\ta^{*T}\matY \matU\matSig \matV^T} \\
\label{eqn: thm_eqn4} &=& \frac{1}{2} \TNorm{\matE} \TNormS{\ta^{*T}\matY\matX}.
\end{eqnarray}
Combining eqns.~\eqref{eqn: thm_eqn3} and~\eqref{eqn: thm_eqn4}, we get
\begin{eqnarray}
\label{eqn:thm_eqn5}
Z_{opt} &\geq & \tilde{Z}_{opt} - \frac{1}{2} \TNorm{\matE} \TNormS{\ta^{*T}\matY\matX}.
\end{eqnarray}
We now proceed to bound the second term in the right-hand side of the above equation. Towards that end, we bound the difference:
\begin{eqnarray*}
&& \lefteqn{\abs{\ta^{*T} \matY\matX\matR\matR^T\matX^T\matY\ta^* - \ta^{*T}\matY\matX\matX^T\matY\ta^*}} \nonumber \\
&=& \abs{\ta^{*T} \matY\matU\matSig\left(\matV^T\matR\matR^T\matV-\matV^T\matV\right)\matSig\matU^T\matY\ta^*} \nonumber \\
&=& \abs{\ta^{*T} \matY\matU\matSig\left(-\matE\right)\matSig\matU^T\matY\ta^*} \nonumber \\
&\leq& \TNorm{\matE}\TNormS{\ta^{*T} \matY\matU\matSig}\nonumber \\
&=& \TNorm{\matE}\TNormS{\ta^{*T} \matY\matU\matSig\matV^T} \nonumber \\
\label{eqn:thm_eqn6} &=& \TNorm{\matE}\TNormS{\ta^{*T} \matY\matX}.
\end{eqnarray*}
We can rewrite the above inequality as $\abs{\TNormS{\ta^{*T}\matY\matX\matR}- \TNormS{\ta^{*T}\matY\matX}} \leq \TNorm{\matE}\TNormS{\ta^{*T} \matY\matX}$; thus,
$$
\TNormS{\ta^{*T}\matY\matX} \leq \frac{1}{1-\TNorm{\matE}} \TNormS{\ta^{*T} \matY\matX \matR}.
$$
Combining with eqn.~\eqref{eqn:thm_eqn5}, we get
\begin{equation}\label{eqn:thm_eqn9}
Z_{opt} \geq  \tilde{Z}_{opt}-\frac{1}{2}\left(\frac{\TNorm{\matE}}{1-\TNorm{\matE}}\right)\TNormS{\ta^{*T} \matY\matX\matR}.
\end{equation}
Now recall from our discussion in Section~\ref{sec:intro} that $\w^{*T} = \a^{*T}\matY\matX$, $\tw^{*T} = \ta^{*T}\matY\matX\matR$, $\TNormS{\w^*} = \sum_{i=1}^n \alpha_i^*$, and $\TNormS{\tw^*} = \sum_{i=1}^n \tilde\alpha_i^*$. Then, the optimal solutions $Z_{opt}$ and $\tilde{Z}_{opt}$ can be expressed as follows:
\begin{eqnarray}
\label{eqn:pd1} Z_{opt} &=& \TNormS{\w^*} -\frac{1}{2}\TNormS{\w^*} = \frac{1}{2}\TNormS{\w^*},\\
\label{eqn:pd2} \tilde{Z}_{opt} &=& \TNormS{\tw^*} -\frac{1}{2}\TNormS{\tw^*} = \frac{1}{2}\TNormS{\tw^*}.
\end{eqnarray}
Combining eqns.~(\ref{eqn:thm_eqn9}),~(\ref{eqn:pd1}), and~(\ref{eqn:pd2}), we get
\begin{eqnarray}
\TNormS{\w^*} &\geq & \TNormS{\tw^*} - \left( \frac{\TNorm{\matE}}{1 - \TNorm{\matE}} \right)\TNormS{\tw^*} \nonumber \\
&=& \left(1-\frac{\TNorm{\matE}}{1 - \TNorm{\matE}} \right)\TNormS{\tw^*}.
\end{eqnarray}

Let $\gamma^*=\norm{\w^*}_2^{-1}$ be the geometric margin of the problem of eqn.~(\ref{eqn:svm1A}) and let $\tilde{\gamma}^*=\norm{\tw^*}_2^{-1}$ be the geometric margin of the problem of eqn.~(\ref{eqn:svm2}). Then, the above equation implies:
\begin{eqnarray}
\gamma^{*2} &\leq& \left(1-\frac{\TNorm{\matE}}{1 - \TNorm{\matE}} \right)^{-1}\tilde{\gamma}^{*2} \nonumber \\
\Rightarrow \tilde{\gamma}^{*2} &\geq&  \left(1-\frac{\TNorm{\matE}}{1 - \TNorm{\matE}} \right)\gamma^{*2}.
\end{eqnarray}
\remove{
In order to conclude the proof of Theorem~\ref{thm:main_result}, we prove a lemma that can be used to bound $\TNorm{E}$ (see Appendix for a short proof).
\begin{lemma} \label{lem:lem_R2}
Let $\epsilon \in (0,1/2]$ be an accuracy parameter and let $\matR \in \mathbb{R}^{d \times r}$ be the random projection matrix of Algorithm~\ref{alg:algfht}, with $r$ as in eqn.~(\ref{eqn:rval}). Then, with probability at least $.95-\delta$,
$$\TNorm{\matE} = \TNorm{\matV^T \matV - \matV^T \matR\matR^T\matV} \leq \epsilon. $$
\end{lemma}
Combining eqn.~(\ref{eqn:pd3}) with Lemma~\ref{lem:lem_R2}, setting $\delta = .05$, and using $\epsilon \leq 1/2$ concludes the proof.
}
\end{proof}

Our second theorem argues that the radius of the minimum ball enclosing all projected points (the rows of the matrix $\matX\matR$) is very close to the radius of the minimum ball enclosing all original points (the rows of the matrix $\matX$). We will prove this theorem for $\matR=\matRhad$ as in eqn.~(\ref{eqn:had}), but similar results can be proven for the other two constructions for $\matR$.

\begin{theorem}\label{thm:second_result}
Fix $\epsilon \in (0,\frac12]$, $\delta \in (0,1]$. Let $B$ be the radius of the minimum ball enclosing all points in the full-dimensional space (the rows of the matrix $\matX$), and let $\tilde{B}$ be the radius of the ball enclosing all points in the dimensionally reduced space (the rows of the matrix $\matX\matR$). Then, if \math{r=O(\rho\epsilon^{-2}\cdot \log(\rho d\delta^{-1})\cdot \log(\rho\epsilon^{-2}\delta^{-1}\log(\rho d\delta^{-1})))}, with probability at least $1-\delta$,
$$\tilde{B}^2 \leq (1+\epsilon) B^2.$$
\end{theorem}
\begin{proof}
We consider the matrix $\matX_B \in \mathbb{R}^{(n+1)\times d}$ whose first $n$ rows are the rows of $\matX$ and whose last row is the vector $\x_B^T$; here $\x_B$ denotes the center of the minimum radius ball enclosing all $n$ points. Then, the SVD of $\matX_B$ is equal to $\matX_B = \matU_B \matSig_B \matV_B^T$, where $\matU_B \in \mathbb{R}^{(n+1) \times \rho_B}$, $\matSig_B \in \mathbb{R}^{\rho_B \times \rho_B}$, and $\matV \in \mathbb{R}^{d \times \rho_B}$. Here $\rho_B$ is the rank of the matrix $\matX_B$ and clearly $\rho_B \leq \rho + 1$. (Recall that $\rho$ is the rank of the matrix $\matX$.) Let $B$ be the radius of the minimal radius ball enclosing all $n$ points in the original space. Then, for any $i=1,\ldots,n$,
\begin{equation}\label{eqn:pd4}
B^2 \geq \TNormS{\x_i-\x_B} = \TNormS{\left(\e_i-\e_{n+1}\right)^T\matX_B}.
\end{equation}
Now consider the matrix $\matX_B\matR$ and notice that
\begin{eqnarray*}
 &&\abs{\TNormS{\left(\e_i-\e_{n+1}\right)^T\matX_B} - \TNormS{\left(\e_i-\e_{n+1}\right)^T\matX_B\matR}}\\
 &=& \abs{\left(\e_i-\e_{n+1}\right)^T\left(\matX_B\matX_B^T-\matX_B\matR\matR^T\matX_B^T\right)\left(\e_i-\e_{n+1}\right)}\\
%
%
&=& \abs{\left(\e_i-\e_{n+1}\right)^T\matU_B\matSig_B\matE_B\matSig_B\matU_B^T\left(\e_i-\e_{n+1}\right)}\\
&\leq& \TNorm{\matE_B}\TNormS{\left(\e_i-\e_{n+1}\right)^T\matU_B\matSig_B}\\
&=& \TNorm{\matE_B}\TNormS{\left(\e_i-\e_{n+1}\right)^T\matU_B\matSig_B\matV_B^T}\\
&=& \TNorm{\matE_B}\TNormS{\left(\e_i-\e_{n+1}\right)^T\matX_B}.
\end{eqnarray*}
In the above, we let $\matE_B \in \mathbb{R}^{\rho_B \times \rho_B}$ be the matrix that satisfies $\matV_B^T\matV_B = \matV_B^T \matR\matR^T \matV_B + \matE_B$, and we also used $\matV_B^T\matV_B = \matI$. Now consider the ball whose center is the $(n+1)$-st row of the matrix $\matX_B\matR$ (essentially, the projection of the center of the minimal radius enclosing ball for the original points). Let $\tilde{i} = \arg \max_{i=1\ldots n} \TNormS{\left(\e_i-\e_{n+1}\right)^T\matX_B\matR}$; then, using the above bound and eqn.~(\ref{eqn:pd4}), we get
\begin{eqnarray*}
\TNormS{\left(\e_{\tilde{i}}-\e_{n+1}\right)^T\matX_B\matR} &\leq& \left(1+\TNorm{\matE_B}\right)\TNormS{\left(\e_{\tilde{i}}-\e_{n+1}\right)^T\matX_B} \\
 &\leq& \left(1+\TNorm{\matE_B}\right)B^2.
\end{eqnarray*}
Thus, there exists a ball centered at $\e_{n+1}^T\matX_B\matR$ (the projected center of the minimal radius ball in the original space) with radius at most $\sqrt{1+\TNorm{\matE_B}}B$ that encloses all the projected points. Recall that $\tilde{B}$ is defined as the radius of the minimal radius ball that encloses all points in projected subspace; clearly,
$$\tilde{B}^2 \leq \left(1+\TNorm{\matE_B}\right)B^2.$$
We can now use Lemma~\ref{lem:spectral} on $\matV_B$ to conclude that (using $\rho_B \leq \rho + 1$) $\TNorm{\matE_B} \leq \epsilon$ 
\end{proof}
Similar theorems can be proven for the two other constructions of $\matR$ by using appropriate values for $r$. We are now ready to conclude the proof of Theorem~\ref{thm:genthm}.
\begin{proof}\textit{(of Theorem~\ref{thm:genthm})}
The proof of Theorem~\ref{thm:genthm} follows by combining Theorem~\ref{thm:main_technical}, Lemma~\ref{lem:spectral}, and Theorem~\ref{thm:second_result}. The failure probability is at most $2\delta$, by a simple application of the union bound.
\end{proof}

Finally, we state the margin preservation theorem for SVM regression, which is analogous to Theorem~\ref{thm:main_technical}. This theorem holds for all four choices of the random projection matrix $\matR \in \mathbb{R}^{d \times r}$ and is identical to the proof of Theorem~\ref{thm:main_technical}.
\begin{theorem}\label{thm:svmr}
Let $\epsilon$ be an accuracy parameter and let \math{\matR\in\mathbb{R}^{d\times r}} be a matrix satisfying
\math{\norm{\matV^T\matV - \matV^T \matR\matR^T \matV}_2\le\epsilon}. Let \math{\gamma^*} and \math{\tilde\gamma^*} be the margins obtained by
solving the SVM regression problems using data matrices \math{\matX} and \math{\matX \matR}
respectively (eqns.~(\ref{eqn:svm4A}) and~(\ref{eqn:svm5})). Then,
$$\tilde\gamma^{*2}\ge \left(1-\epsilon\right)\cdot \gamma^{*2}.$$
\end{theorem}

\section{Experiments}\label{sec:exp}
In our experimental evaluations, we implemented random projections using four different methods: RG, RS, FHT, and CW (see Section~\ref{subsec:prel} for definitions) in MATLAB version 7.13.0.564 (R2011b).  We ran the algorithms using the same values of $r$  (the dimension of the projected feature space) for all algorithms, but we varied $r$ across different datasets. We used LIBLINEAR \cite{Fan08} and LIBSVM \cite{Chang11} as our linear SVM solver with default settings. In all cases, we ran our experiments on the original full data (referred to as ``full'' in the results), as well as on the projected data. For large-scale datasets, we use LIBLINEAR which is a faster SVM solver than LIBSVM, while for medium-scale datasets we use LIBSVM.
We partitioned the data randomly for ten-fold cross-validation in order to estimate out-of-sample error. We repeated this partitioning ten times to get ten ten-fold cross-validation experiments. In the case where the dataset is already available in the form of a training and test-set, we do not perform ten-fold cross validation and use the given training and test set instead. In order to estimate the effect of the randomness in the construction of the random projection matrices, we repeated our cross-validation experiments ten times using ten different random projection matrices for all datasets. For classification experiments, we report in-sample error ($\epsilon_{in}$), out-of-sample error ($\epsilon_{out}$), the time to compute random projections ($t_{rp}$), the total time needed to both compute random projections \textit{and} run SVMs on the lower-dimensional problem ($t_{run}$), and the margin ($\gamma$). For regression experiments, we report the margin, the combined running-time of random projections and SVM, mean-squared error $(mse)$ and the squared correlation-coefficient $(\beta)$ of $\epsilon_{in}$. All results are averaged over the ten cross-validation experiments and the ten choices of random projection matrices. For each of the aforementioned quantities, we report both its mean value $\mu$ and its standard deviation $\sigma$. 

\subsection{Experiments on SVM Classification}
We describe experimental evaluations on three real-world datasets, namely a collection of document-term matrices (the TechTC-300 dataset \cite{David04}), a subset of the Reuters Corpus dataset (RCV1 Dataset \cite{DL04b}) and a population genetics dataset (the joint Human Genome Diversity Panel or HGDP \cite{Li08} and the HapMap Phase 3 data \cite{Pasch10}) and also on three synthetic datasets. The synthetic datasets, a subset of the RCV1 dataset and the TechTC-300 dataset correspond to binary classification tasks while the joint HapMap-HGDP dataset and a subset of the RCV1 dataset correspond to multi-class classification tasks; our algorithms perform well in multi-class classification as well. For the multi-class experiments of Section~\ref{subsubsec:snp}, we do not report a margin. We use LIBLINEAR as our SVM solver for Hapmap-HGDP \footnote{In \citeN{PBMD13}, the experiments on Hapmap-HGDP dataset were done using LIBSVM's one-against-one multi-class classification method. LIBLINEAR does not have the one-against-one method implemented in the package. So we use the Crammer and Singer method  \cite{CS00}.} and the RCV1 datasets, while for the remaining datasets we use LIBSVM as our solver. For multi-class experiments, we use the method of Crammer and Singer \cite{CS00} implemented in LIBLINEAR.
\begin{table}[!ht]
\tbl{$\epsilon_{out}$ and $\gamma$ of Synthetic Data \label{tab:tabsyn}}{
\begin{tabular}{|c||c|c|c|c|c|}
\hline
	\multicolumn{2}{|c|}{$\epsilon_{out}$}&\multicolumn{3}{c|}{Projected Dimension \math{r}}&\\
	\cline{3-6}
	\multicolumn{2}{|c|}{}{}&256 &512 & 1024 & \textbf{full} \\
	\hline
	\multirow{4}{*}{D1} &CW $(\mu)$ &24.08  &19.45   &16.66  &$\mathbf{15.10}$ \\
				&$(\sigma)$  &4.52    &4.15    &3.52  &$\mathbf{2.60}$ \\ 	
	\cline{2-6}
		                 &RS  $(\mu)$   &24.1.0  &19.46   &16.36  &$\mathbf{15.10}$ \\
				&$(\sigma)$   &4.45    &3.79    &3.22 &$\mathbf{2.60}$ \\		
	\cline{2-6}				
		                 &FHT  $(\mu)$  &23.52   &19.59   &16.67 &$\mathbf{15.10}$ \\
			       &  $(\sigma)$  &4.21    &4.05    &3.37   &$\mathbf{2.60}$ \\
	\cline{2-6}				
		                 &RG  $(\mu)$  & 24.34   &19.73   &16.69  &$\mathbf{15.10}$ \\
			       &  $(\sigma)$  &4.44    &3.86    &3.28   &$\mathbf{2.60}$ \\
	\hline\hline
	\multirow{4}{*}{D2} &CW $(\mu)$  &25.94   &21.07   &17.33   &$\mathbf{15.44}$ \\
				&$(\sigma)$  &4.13   &4.16    &3.45  & $\mathbf{2.54}$ \\ 	
	\cline{2-6}
		                 &RS $(\mu)$   &25.80   &20.80   &17.47   &$\mathbf{15.44}$ \\
				&$(\sigma)$   &4.40    &3.93    &3.42    &$\mathbf{2.54}$ \\		
	\cline{2-6}				
		                 &FHT  $(\mu)$   &25.33   &21.23   &17.58   &$\mathbf{15.44}$\\
			       &$(\sigma)$     &3.69    &4.24    &3.53    &$\mathbf{2.54}$\\
	\cline{2-6}				
		                 &RG  $(\mu)$   &25.43   &20.54   &17.25   &$\mathbf{15.44}$\\
			       &$(\sigma)$     &4.03    &3.65    &3.38    &$\mathbf{2.54}$\\
	\hline\hline
	\multirow{4}{*}{D3} &CW $(\mu)$ &27.62 &22.97 &18.93 &$\mathbf{15.83}$ \\
				&$(\sigma)$ &3.46 &3.22 &3.32 &$\mathbf{2.00}$ \\ 	
	\cline{2-6}
		                 &RS $(\mu)$  &28.15 &23.00 &18.72 &$\mathbf{15.83}$ \\
				& $(\sigma)$  &3.02 &3.48 &2.78 &$\mathbf{2.00}$ \\		
	\cline{2-6}				
		                 &FHT $(\mu)$  &27.92 &23.41 &18.73 &$\mathbf{15.83}$ \\
			       & $(\sigma)$     &3.46   &3.60  &3.02 &$\mathbf{2.00}$ \\
	\cline{2-6}				
		                 &RG $(\mu)$  &27.71  &22.85  &18.96  &$\mathbf{15.83}$ \\
			       & $(\sigma)$     &3.38    &3.29   &3.33  &$\mathbf{2.00}$ \\
	\hline	
\end{tabular}
\quad
\begin{tabular}{|c||c|c|c|c|c|}
\hline
	\multicolumn{2}{|c|}{$\gamma$}&\multicolumn{3}{c|}{Projected Dimension \math{r}}&\\
	\cline{3-6}
	\multicolumn{2}{|c|}{}{}&256 &512 & 1024 & \textbf{full} \\
	\hline
	\multirow{4}{*}{D1} &CW $(\mu)$ &5.72    &6.67    &7.16  &$\mathbf{7.74}$ \\
				&$(\sigma)$  &0.58    &0.58    &0.59    &$\mathbf{0.59}$ \\ 	
	\cline{2-6}
		                 &RS  $(\mu)$    &5.73    &6.66    &7.18    &$\mathbf{7.74}$ \\
				&$(\sigma)$   &0.57    &0.55    &0.55   &$\mathbf{0.59}$ \\		
	\cline{2-6}				
		                 &FHT  $(\mu)$  &5.76    &6.64    &7.15    &$\mathbf{7.74}$ \\
			       &  $(\sigma)$  &0.56    &0.58    &0.56  &$\mathbf{0.59}$ \\
	\cline{2-6}				
		                 &RG  $(\mu)$  &5.67    &6.60    &7.13    &$\mathbf{7.74}$ \\
			       &  $(\sigma)$  &0.57    &0.51    &0.54   &$\mathbf{0.59}$ \\
	\hline\hline
	\multirow{4}{*}{D2} &CW $(\mu)$    &6.62    &8.09    &8.88  &$\mathbf{9.78}$ \\
				&$(\sigma)$  &0.64    &0.62    &0.59  &$\mathbf{0.66}$ \\ 	
	\cline{2-6}
		                 &RS $(\mu)$   &6.65    &8.10    &8.88   &$\mathbf{9.78}$ \\
				&$(\sigma)$    &0.64    &0.60    &0.63  &$\mathbf{0.66}$ \\		
	\cline{2-6}				
		                 &FHT  $(\mu)$  &6.66    &8.06    &8.84   &$\mathbf{9.78}$\\
			       &$(\sigma)$     &0.63    &0.65    &0.63   &$\mathbf{0.66}$\\
	\cline{2-6}				
		                 &RG  $(\mu)$  &6.66    &8.13    &8.90   &$\mathbf{9.78}$\\
			       &$(\sigma)$     &0.65    &0.60    &0.63   &$\mathbf{0.66}$\\
	\hline\hline
	\multirow{4}{*}{D3} &CW $(\mu)$ &7.69    &9.84   &11.07 &$\mathbf{12.46}$ \\
				&$(\sigma)$  &0.67    &0.60    &0.71  &$\mathbf{0.69}$ \\ 	
	\cline{2-6}
		                 &RS $(\mu)$ &7.61  &9.85   &11.05  &$\mathbf{12.46}$ \\
				& $(\sigma)$  &0.59    &0.6212    &0.62   &$\mathbf{0.69}$ \\		
	\cline{2-6}				
		                 &FHT $(\mu)$  &7.63    &9.83   &11.11 &$\mathbf{12.46}$ \\
			       & $(\sigma)$    &0.67    &0.64    &0.64 &$\mathbf{0.69}$ \\
	\cline{2-6}				
		                 &RG $(\mu)$  &7.69    &9.85   &11.04  &$\mathbf{12.46}$ \\
			       & $(\sigma)$    &0.67    &0.61    &0.7  &$\mathbf{0.69}$ \\

	\hline	
\end{tabular}
}
\begin{tabnote}
\Note{Synthetic data:}{\math{\epsilon_{out}} decreases and \math{\gamma} increases as a function of $r$ in all three families of matrices, using any of the four random projection methods. $\mu$ and $\sigma$ indicate the mean and the standard deviation of $\epsilon_{out}$ over ten matrices in each family $D1$, $D2$, and $D3$, ten ten-fold cross-validation experiments, and ten choices of random projection matrices for the four methods that we investigated (a total of 1,000 experiments for each family of matrices).}
\end{tabnote}
\end{table}	

\subsubsection{Synthetic datasets}\label{subsubsec:synth}
The synthetic datasets are separable by construction. More specifically, we first constructed a weight vector $\w \in \mathbb{R}^{d}$, whose entries were selected in i.i.d. trials from a Gaussian distribution ${\mathcal N}(\mu, \sigma)$ of mean $\mu$ and standard-deviation $\sigma$. We experimented with the following three distributions: ${\mathcal N}(0,1)$, ${\mathcal N}(1,1.5)$, and ${\mathcal N}(2,2)$. Then, we normalized $\w$ to create $\hw = \w / \TNorm{\w}$. Let $ \matX_{ij} = {\mathcal N}(0,1)$; then, we set $\x_i$ to be equal to the $i$-th row of $\matX$, while $\y_i = sign \left( \hw^T \x_i \right) $.
We generated families of matrices of different dimensions. More specifically, family $D1$ contained matrices in $\mathbb{R}^{200\times 5,000}$; family $D2$ contained matrices in $\mathbb{R}^{250\times 10,000}$; and family $D3$ contained matrices in $\mathbb{R}^{300\times 20,000}$. We generated ten datasets for each of the families $D1$, $D2$, and $D3$, and we report average results over the ten datasets. We set $r$ to $256$, $512$, and $1024$ and set $C$ to 1,000 in LIBSVM for all the experiments. Tables ~\ref{tab:tabsyn}  shows  $\epsilon_{out}$ and $\gamma$ for the three datasets $D1$, $D2$, and $D3$. $\epsilon_{in}$ is zero for all three data families. As expected, $\epsilon_{out}$ and $\gamma$ improve as $r$ grows for all four random projection methods. Also, the time needed to compute random projections is very small compared to the time needed to run SVMs on the projected data. Figure~\ref{fig:syn2} shows the combined running time of random projections and SVMs, which is nearly the same for all four random projection methods. It is obvious that this combined running time is much smaller that the time needed to run SVMs on the full dataset (without any dimensionality reduction). For instance, for $r=1024$, $t_{run}$ for $D1$, $D2$, and $D3$  is (respectively) 6, 9, and 25 times smaller than $t_{run}$ on the full-data. 
\begin{figure}[!tp]

\begin{center}
\includegraphics[height = 60mm,width=\columnwidth,clip]{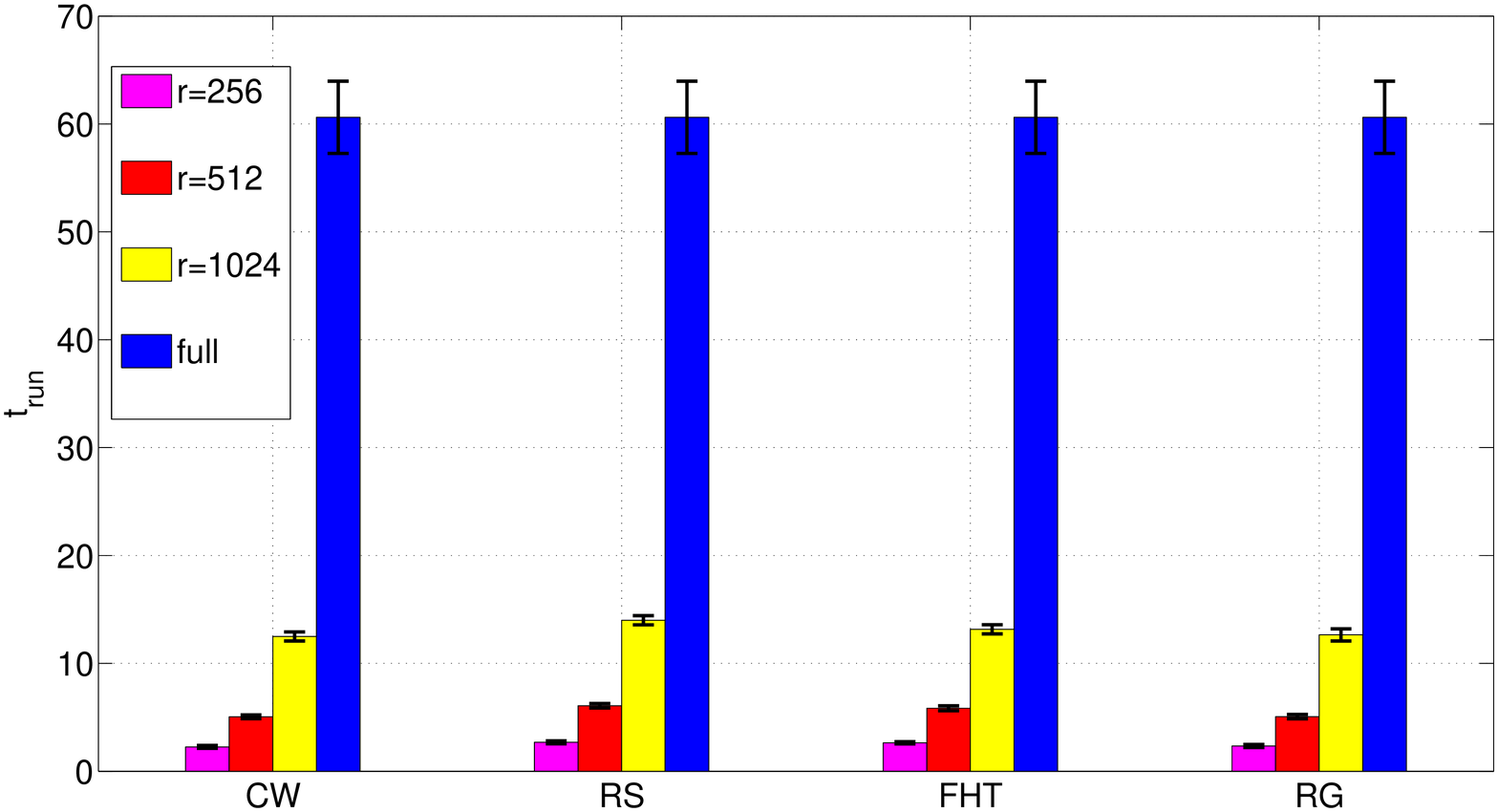}
D1
\includegraphics[height = 60mm,width=\columnwidth,clip]{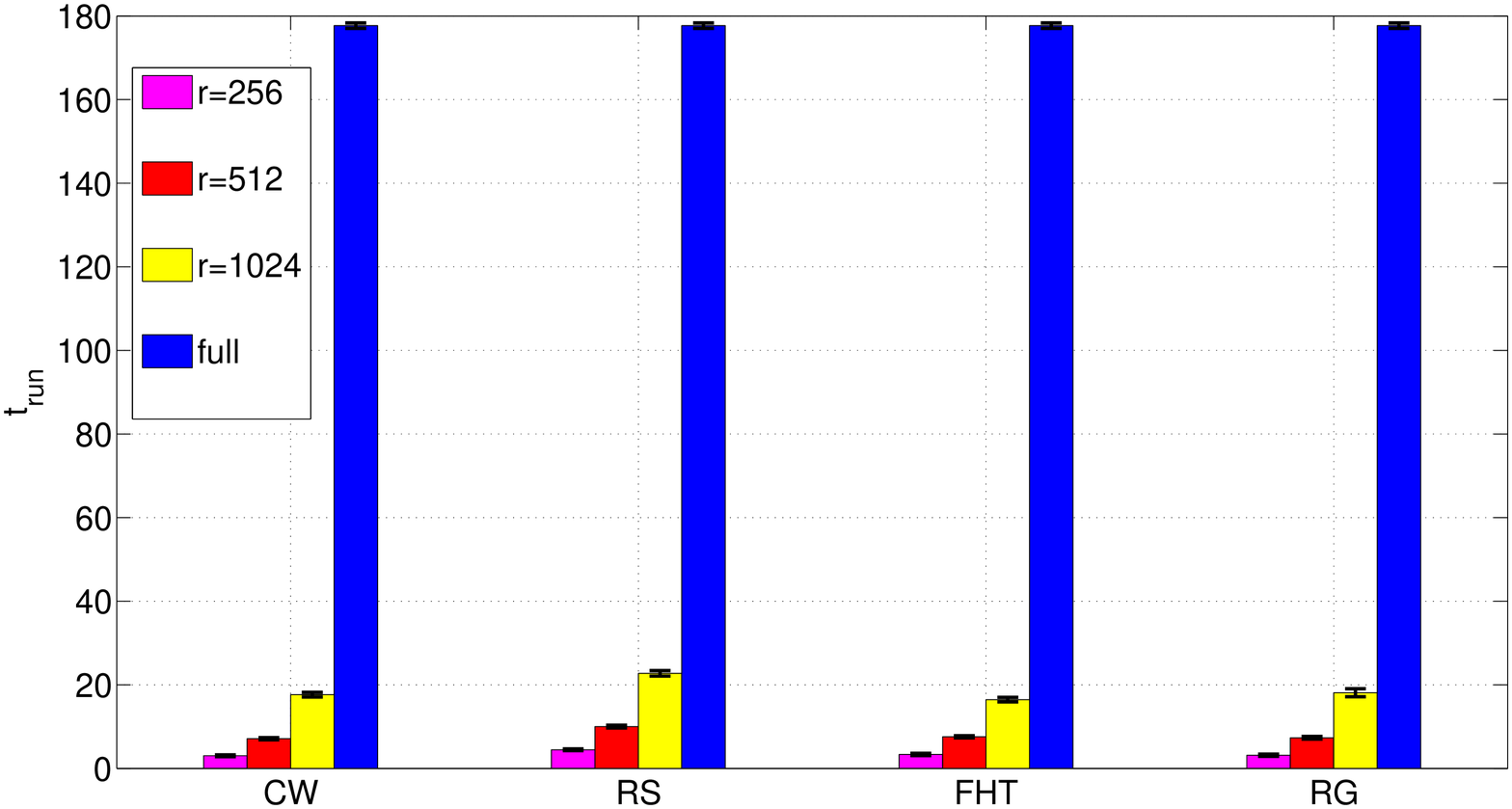}
D2
\includegraphics[height = 60mm,width=\columnwidth,clip]{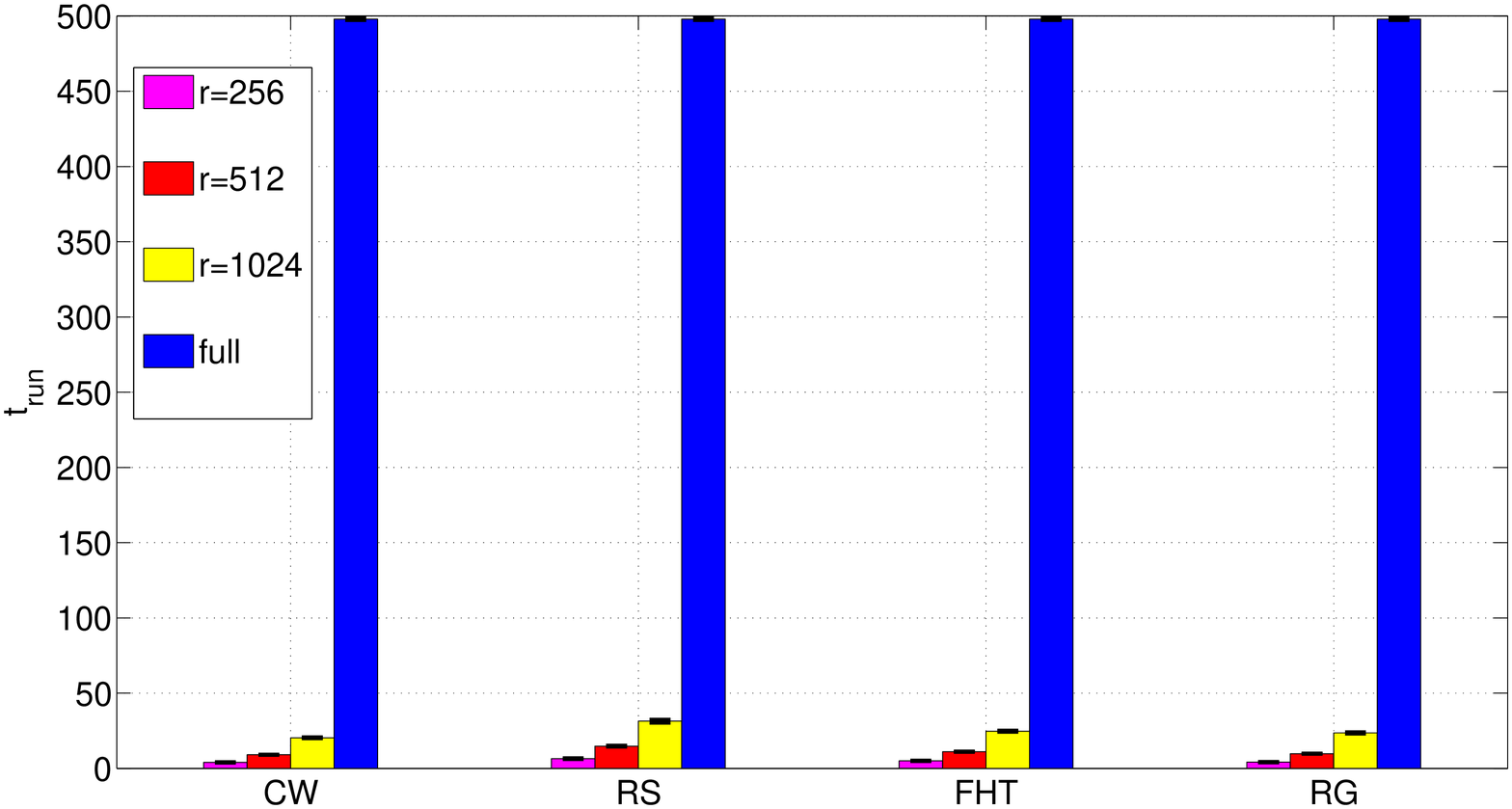}
D3
\caption{\small Total (average) running times, in seconds, of random projections \textit{and} SVMs  on the lower-dimensional data for each of the three families of synthetic data. Vertical bars indicate the, relatively small, standard deviation (see the caption of Table ~\ref{tab:tabsyn}).}
\label{fig:syn2}
\end{center}
\end{figure}

\subsubsection{The TechTC-300 dataset}\label{subsubsec:techtc}

For our first real dataset, we use the TechTC-300 data, consisting of a family
of 295 document-term data matrices. The TechTC-300 dataset comes from the Open Directory Project (ODP), which is a large, comprehensive directory of the web, maintained by volunteer editors. Each matrix in the TechTC-300 dataset contains a pair of categories from the ODP. Each category corresponds to a label, and thus the resulting classification task is binary. The documents that are collected from the union of all the subcategories within each category are represented in the bag-of-words model, with the words constituting the features of the data \cite{David04}. Each data matrix consists of 150-280 documents (the rows of the data matrix $\matX$), and each document is described with respect to 10,000-40,000 words (features, columns of the matrix $\matX$). Thus, TechTC-300 provides a diverse collection of data sets for a
systematic study of the performance of the SVM on the projected
versus full data.
\begin{table}[ht]
\tbl{TechTC-300 Dataset \label{tab:tabtechtc}}{

\begin{tabular}{|c||c|c|c|c|c|}
\hline
	\multicolumn{2}{|c|}{}&\multicolumn{3}{c|}{Projected Dimension \math{r}}&\\
	\cline{3-6}
	\multicolumn{2}{|c|}{}{}&128 &256 &512 & \textbf{full} \\
	\hline
	\multirow{4}{*}{$\mathbf{\epsilon_{out}} $} &CW$(\mu)$ &24.63 &22.84 &21.26  &$\mathbf{17.35}$  \\
							      &$(\sigma)$ &10.57 &10.37 &10.17  &$\mathbf{9.45}$\\
	\cline{2-6}
		                 &RS$(\mu)$ &24.58 &22.90 &21.38  &$\mathbf{17.35}$  \\
				&$(\sigma)$  &10.57 &10.39 &10.23  &$\mathbf{9.45}$\\
	\cline{2-6}				
		                 &FHT $(\mu)$ &24.63 &22.93 &21.35  &$\mathbf{17.35}$ \\
			& $(\sigma)$  &10.66 &10.39 &10.2  &$\mathbf{9.45}$\\
	\cline{2-6}				
		                 &RG $(\mu)$ &24.59  &22.96  &21.36  &$\mathbf{17.35}$ \\
			& $(\sigma)$  &10.54  &10.5  &10.18  &$\mathbf{9.45}$\\
	\hline\hline
	\multirow{4}{*}{$\mathbf{\gamma}$ } &CW $(\mu)$  &1.66 &1.88 &1.99 &$\mathbf{2.09}$  \\
							&  $(\sigma)$   &3.68 &3.79 &3.92 & $\mathbf{4.00}$ \\
	\cline{2-6}
		                 &RS $(\mu)$   	&1.66	&1.88	&1.99	&$\mathbf{2.09}$ \\
				& $(\sigma )$       &3.65 & 3.80 &3.91 &$\mathbf{4.00}$ \\		
	\cline{2-6}				
		                 &FHT $(\mu)$  	&1.66	&1.88	&1.98	&$\mathbf{2.09}$ \\
				&  $(\sigma)$      &3.65  &3.81 &3.88 &$\mathbf{4.00}$ \\
	\cline{2-6}				
		                 &RG $(\mu)$  	&1.66	&1.88	&1.99	&$\mathbf{2.09}$ \\
				&  $(\sigma)$      &3.70   &3.83   &3.91   &$\mathbf{4.00}$ \\
	\hline\hline
	\multirow{4}{*}{$\mathbf{t_{rp} }$} &CW $(\mu)$  &0.0046 &0.0059 &0.0075 &$\mathbf{--}$  \\
							& $(\sigma)$    &0.0019 &0.0026 &0.0033 &$\mathbf{--}$ \\
	\cline{2-6}
		                 &RS $(\mu)$   	&0.0429  &0.0855	&0.1719	&$\mathbf{--}$ \\
				& $(\sigma)$        &0.0178 &0.0356  &0.072  &$\mathbf{--}$ \\		
	\cline{2-6}				
		                 &FHT $(\mu)$ 	&0.0443  &0.0882	&0.1764	&$\mathbf{--}$ \\
				&  $(\sigma)$       &0.0206  &0.0413 &0.0825  &$\mathbf{--}$ \\
	\cline{2-6}				
		                 &RG $(\mu)$ 	&0.039  &0.078	&0.1567	&$\mathbf{--}$ \\
				&  $(\sigma)$  &0.0159  &0.0318  &0.0642  &$\mathbf{--}$ \\
	\hline\hline
	\multirow{4}{*}{$\mathbf{t_{run} }$} &CW $(\mu)$  &1.23 &2.22 &4.63 &$\mathbf{4.85}$  \\
							& $(\sigma)$    &0.87 &0.93 &1.93 &$\mathbf{2.12}$ \\
	\cline{2-6}
		                 &RS $(\mu)$  	&0.99	&1.53	&3.02	&$\mathbf{4.85}$ \\
				&  $(\sigma )$      &0.97 &0.59  &1.12  &$\mathbf{2.12}$ \\		
	\cline{2-6}				
		                 &FHT $(\mu)$  	&0.95	&1.46	&2.83	&$\mathbf{4.85}$ \\
				&  $(\sigma)$       &0.96  &0.55 &1.02  &$\mathbf{2.12}$ \\
	\cline{2-6}				
		                 &RG $(\mu)$  	&0.82	&1.23	&2.48	&$\mathbf{4.85}$ \\
				&  $(\sigma)$       &0.83  &0.45  &0.84  &$\mathbf{2.12}$ \\

	\hline
\end{tabular}}
\begin{tabnote}
\Note{TechTC300:}{Results on the TechTC300 dataset, averaged over 295 data matrices using four different random projection methods. The table shows how \math{\epsilon_{out}}, \math{\gamma}, \math{t_{rp}} (in seconds), and \math{t_{run}} (in seconds) depend on $r$. $\mu$ and $\sigma$ indicate the mean and the standard deviation of each quantity over 295 matrices, ten ten-fold cross-validation experiments, and ten choices of random projection matrices for the four methods that we investigated.}
\end{tabnote}\
\end{table}	
We set the parameter $C$ to 500 in LIBSVM for all 295 document-term matrices and set $r$ to $128$, $256$, and $512$. We use a lower value of C than for the other data sets for computational reasons: larger \math{C} is less efficient. We note that our classification accuracy is slightly worse (on the full data) than the accuracy presented in Section 4.4 of \citeN{David04}, because we did not fine-tune the SVM parameters as they did, since that is not the focus of this study. For every dataset and every value of \math{r} we tried, the in-sample error on the projected data matched the in-sample error on the full data. We thus focus on \math{\epsilon_{out}}, the margin \math{\gamma}, the time needed to compute random projections \math{t_{rp}},
and the total running time \math{t_{run}}. We report our results averaged over 295 data matrices. Table~\ref{tab:tabtechtc} shows the behavior of these parameters for different choices of \math{r}.
As expected, \math{\epsilon_{out}} and the margin \math{\gamma} improve as \math{r} increases, and they are nearly identical for all four random projection methods. The time needed to compute random projections is smallest for CW, followed by RG, RS and FHT. As a matter of fact, \math{t_{rp}} for CW is ten to 20 times faster than RG, RS and FHT for different values of $r$. This is predicted by the theory in \citeN{Clark12}, since CW is optimized to take advantage of input sparsity. However, this advantage is lost when SVMs are applied on the dimensionally-reduced data. Indeed, the combined running time $t_{run}$ is fastest for RG, followed by FHT, RS and CW.  In all cases, the total running time is smaller than the SVM running time on full dataset. For example, in the case of RG or FHT, setting $r=512$ achieves a running time \math{t_{run}} which is about twice as fast as running SVMs on the full dataset; $\epsilon_{out}$ increases by less than $4\%$.
\begin{figure}[!htb]
\begin{center}
\includegraphics[height = 55mm,width=\columnwidth,clip]{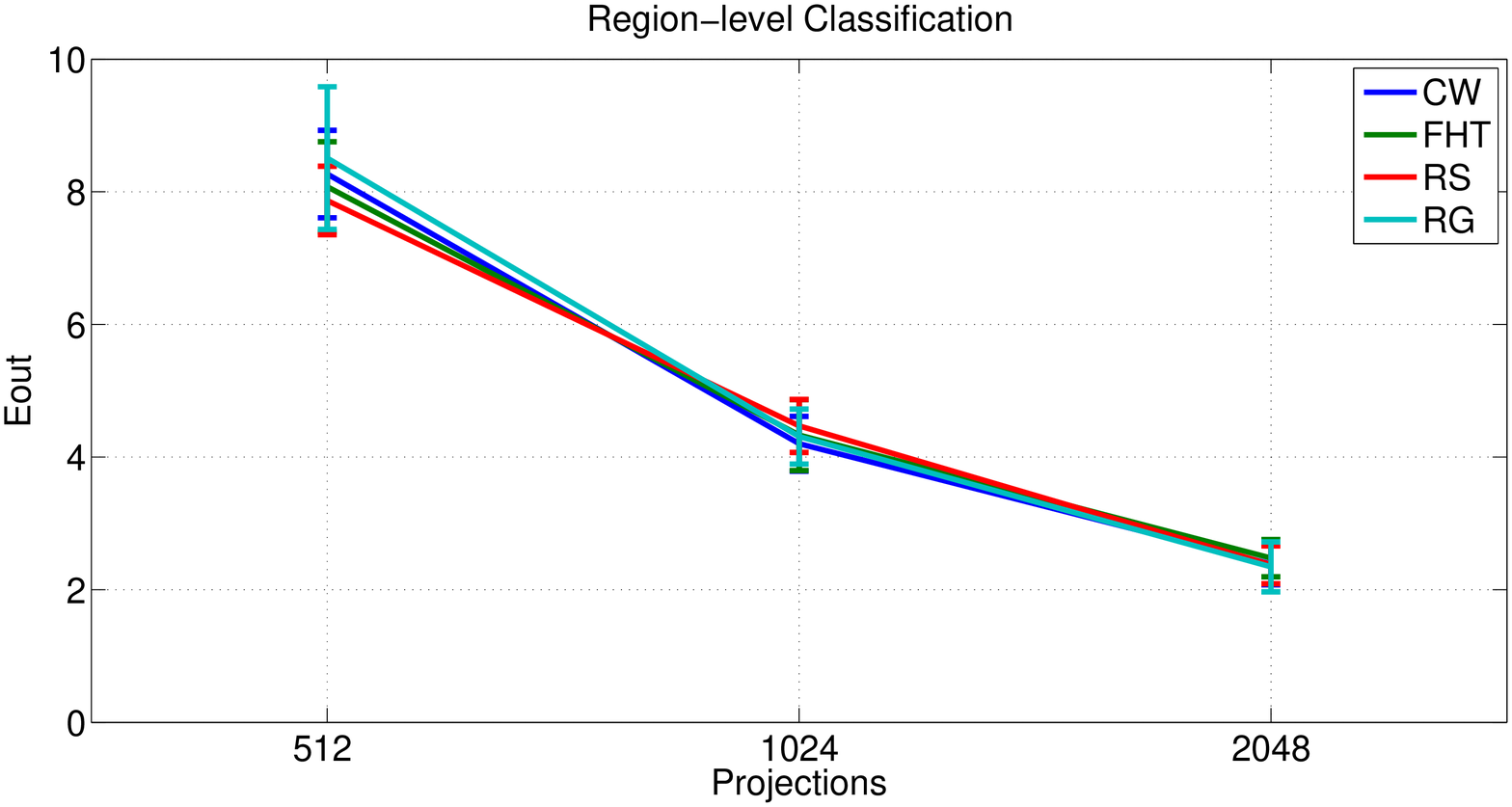}
\includegraphics[height = 55mm,width=\columnwidth,clip]{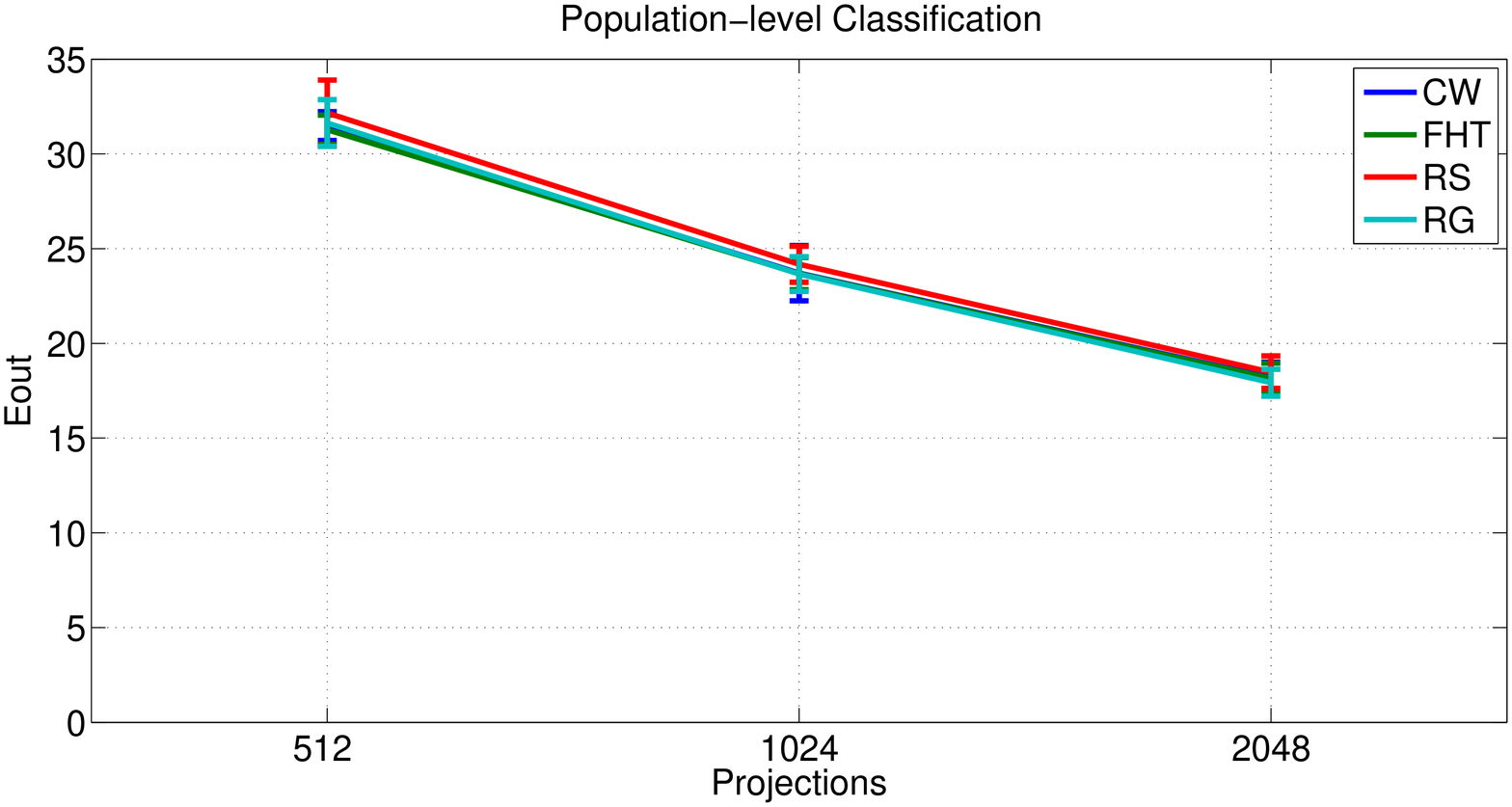}
\caption{\small{$\epsilon_{out}$ as a function of $r$ in the Hapmap-HGDP dataset for four different random projection methods and two different classification tasks. Vertical bars indicate the standard-deviation over the ten ten-fold cross-validation experiments and the ten choices of the random projection matrices for each of the four methods.}}
\label{fig:snp_eout}
\end{center}
\end{figure}

\subsubsection{The HapMap-HGDP dataset}\label{subsubsec:snp}
Predicting ancestry of individuals using a set of genetic markers is a well-studied classification problem. We use a population genetics dataset from the Human Genome Diversity Panel (HGDP) and the HapMap Phase 3 dataset (see \citeN{Pasch10} for details), in order to classify individuals into broad geographic regions, as well as into
 (finer-scale) populations. We study a total of 2,250 individuals from approximately 50 populations and five broad geographic regions (Asia, Africa, Europe, the Americas, and Oceania).
The features in this dataset correspond to $492,516$ Single Nucleotide Polymorphisms (SNPs), which are well-known biallelic loci of genetic variation across the human genome. Each entry in the resulting $2,250\times 492,516$ matrix is set to $+1$ (homozygotic in one allele), $-1$ (homozygotic in the other allele), or $0$ (heterozygotic), depending on the genotype of the respective SNP for a particular sample. Missing entries were filled in with $-1$, $+1$, or $0$, with probability 1/3. Each sample has a known population and region of origin, which constitute its label.
%
\begin{figure}[!tp]

\begin{center}
\centerline{\includegraphics[height = 55mm,width=\columnwidth,clip]{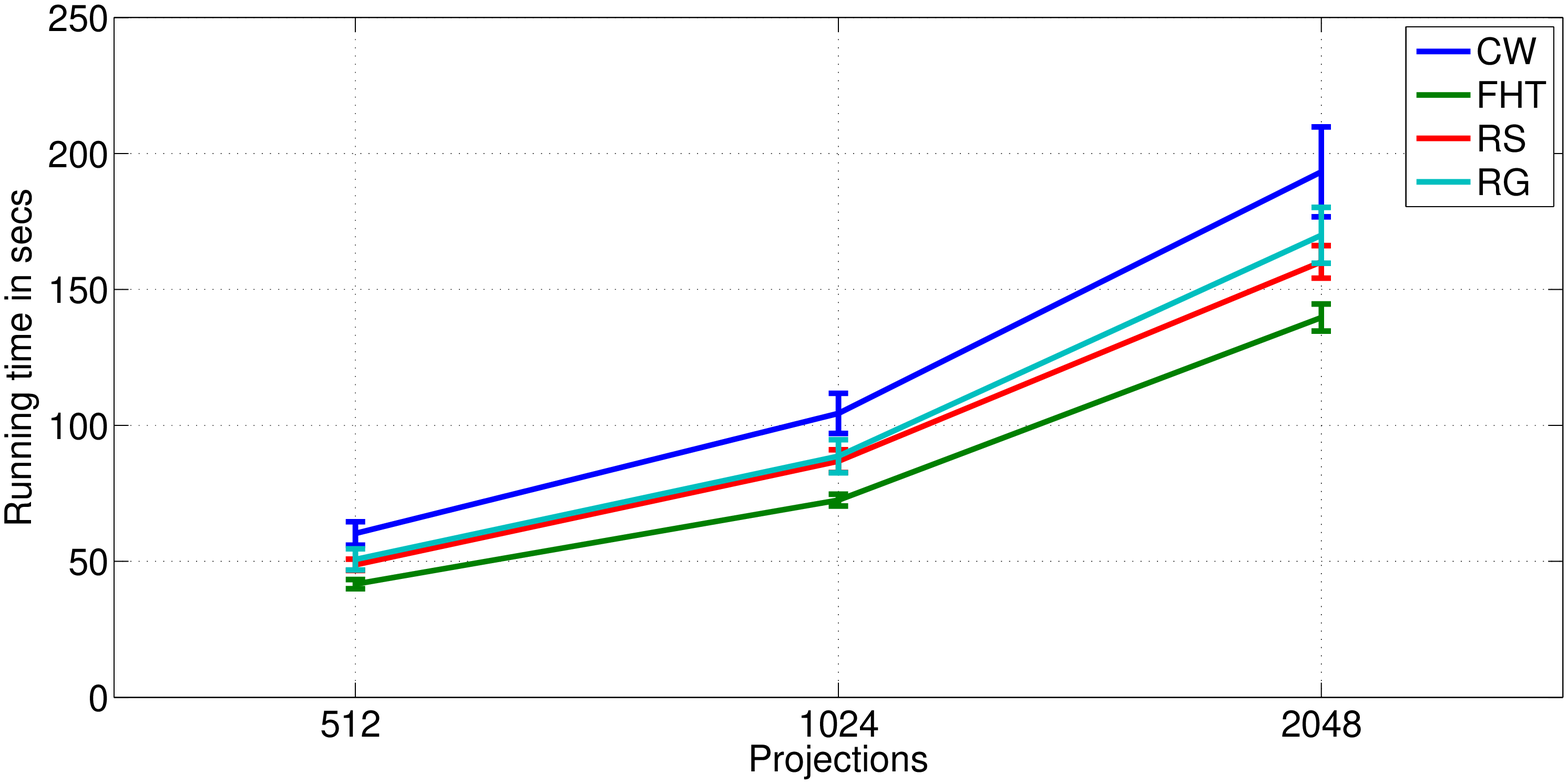}}
Total running time: regional classification
\centerline{\includegraphics[height = 55mm,width=\columnwidth,clip]{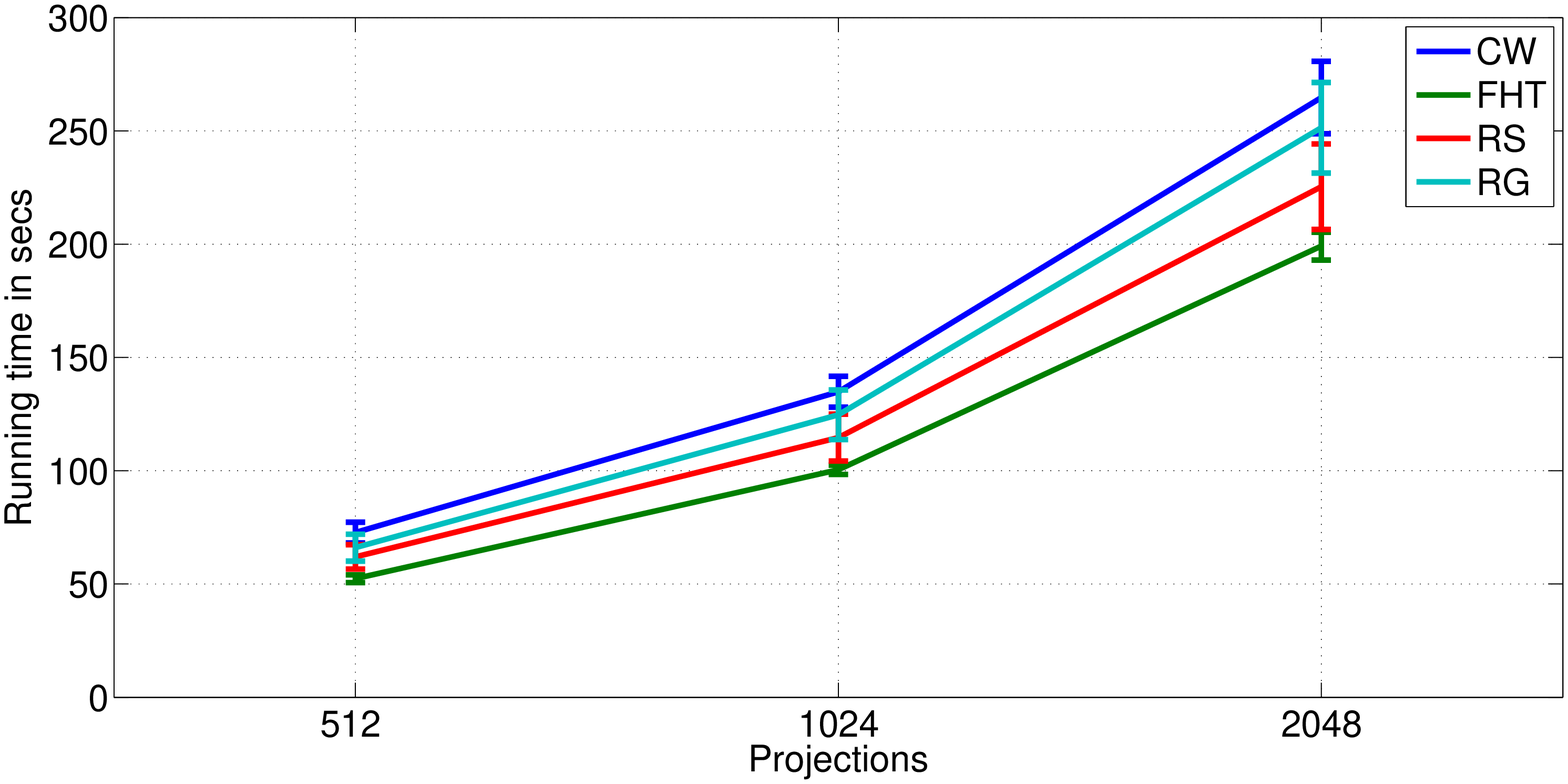}}
Total running time: population-level classification
\centerline{\includegraphics[height = 55mm,width=\columnwidth,clip]{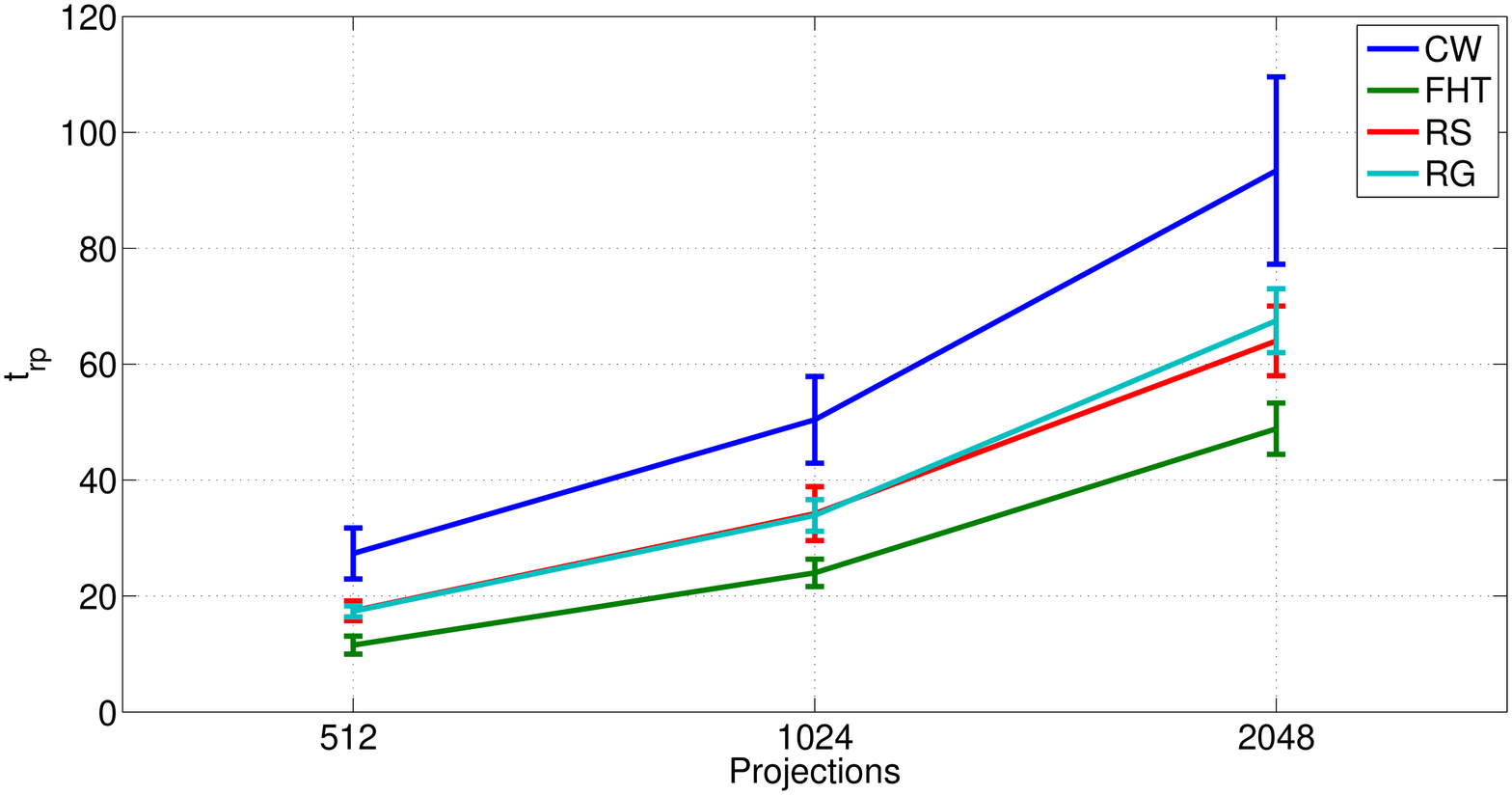}}
Time needed to compute random projections
\caption{\small{Total running time in seconds (random projections \textit{and} SVM classification on the dimensionally-reduced data) for Hapmap-HGDP dataset for four different projection methods using both regional and population-level labels. Notice that the time needed to compute random projection is independent of the classification labels. Vertical bars indicate standard-deviation, as in Figure~\ref{fig:snp_eout}.}}
\label{fig:snp_t}
\end{center}

\end{figure}
We set $r$ to $256$, $512$, $1024$, and $2048$ in our experiments. Since this task is a multi-class classification problem, we used LIBLINEAR's Crammer and Singer technique for classification. We ran two sets of experiments: in the first set, the classification problem is to assign samples to broad regions of origin, while in the second experiment, our goal is to classify samples into (fine-scale) populations. We set $C$ to 1,000 in LIBLINEAR for all the experiments. The in-sample error is zero in all cases. Figure~\ref{fig:snp_eout} shows the out-of-sample error for regions and populations classification, which are nearly identical for all four random projection methods. For regional classification, we estimated $\epsilon_{out}$ to be close to $2\%$, and for population-level classification, $\epsilon_{out}$ is close to $20\%$. This experiment strongly supports the \textbf{computational benefits} of our methods in terms of main memory. $\matX$ is a $2,250\times 492,516$ matrix, which is too large to fit into memory in order to run SVMs. Figure~\ref{fig:snp_t} shows that the combined running time for four different random projection methods are nearly identical for both regions and population classification tasks. However, the time needed to compute the random projections is different from one method to the next. FHT is fastest, followed by RS, RG and CW. In this particular case, the input matrix is dense, and CW seems to be outperformed by the other methods as the running time of CW depends on the number of non-zeros of the matrix. 

\subsubsection{The RCV1 dataset}\label{subsubsec:rcv1}
The RCV1 dataset \cite{DL04b}\footnote{ The RCV1 dataset is available publicly at http://www.csie.ntu.edu.tw/~cjlin/libsvmtools/datasets and contains a training-set and a test-set of predesignated size.} is a benchmark dataset on text categorization. We use the RCV1 dataset for both binary and multi-class classification tasks. The RCV1 binary classification dataset had one training set containing 20,242 data-points and 47,236 features. We generate ten different test-sets each containing 20,000 points and 47,236 features from the available test-set for the binary classification task. The RCV1 binary classification dataset contains CCAT and ECAT as the positive classes and GCAT and MCAT as the negative classes. Instances in both positive and negative classes are absent. The RCV1 multi-class dataset had one training set 15,564 training points with 47,236 features and ten different test-sets each containing 16,000 data points and 47,236 features were generated from the available test-set. There are 53 classes in the RCV1 multi-class dataset. We set $r$ to $2048$, $4096$ and $8192$. We use LIBLINEAR as our SVM solver. We use the L2-regularized L2-loss support vector classification in the primal mode with default settings for binary classification and the method of Crammer and Singer \cite{CS00} for multi-class classification. We set $C=10$ for both multi-class classification and binary classification. The RCV1 dataset is very sparse with 0.16\% non-zero entries in the binary-class dataset and 0.14\% non-zero entries in the multi-class dataset. 

Tables ~\ref{tab:rcv1_multiclass} and ~\ref{tab:rcv1_binaryclass} show the results for RCV1 binary and multi class datasets.$t_{svm}$ denotes the SVM-training time on the projected data. For both binary and multi-class tasks, we observe that $\epsilon_{out}$ is close to that of full dataset and $\epsilon_{out}$ decreases with increase in number of projections. The SVM training time is smaller than that of full dataset for CW method only. For the other methods like FHT, RS and RG, the number of non-zeros in the projected data increases which increases the SVM training time. This is evident from Figure~\ref{fig:rcv1_nnz} which shows the ratio of number of non-zeros of projected data and full data. For all methods except CW, the number of non-zeros increases with increase in value of $r$. This follows from the theory predicted in \citeN{Clark12}, since CW method takes advantage of input sparsity. The combined running time is smaller than that of full-dataset for CW method. The margin of the projected data is close to that of full data for RCV1 binary class dataset.
\begin{figure}[!tp]
\includegraphics[height = 40mm,width=0.5 \columnwidth,clip]{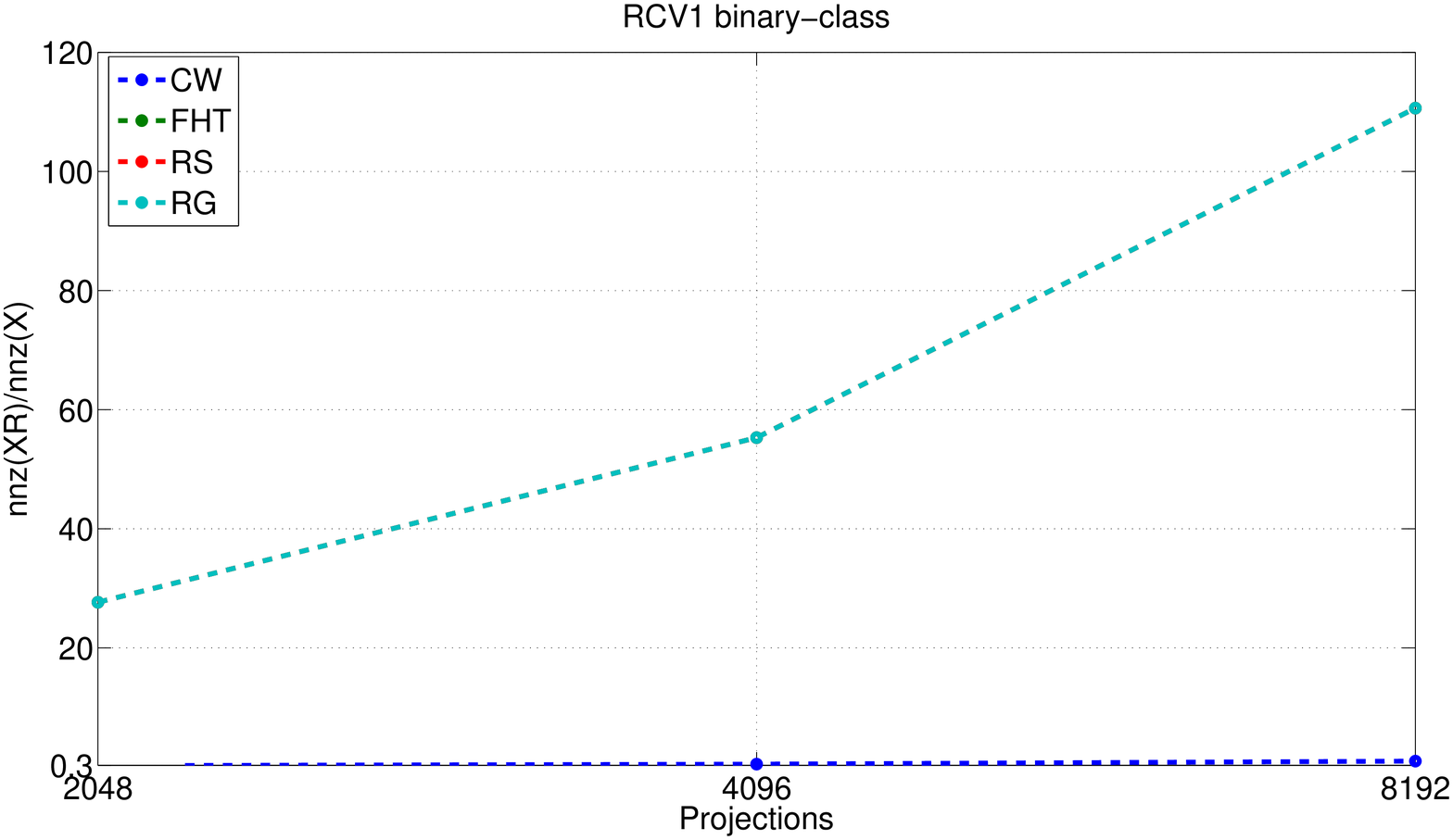}
\includegraphics[height = 40mm,width=0.5 \columnwidth,clip]{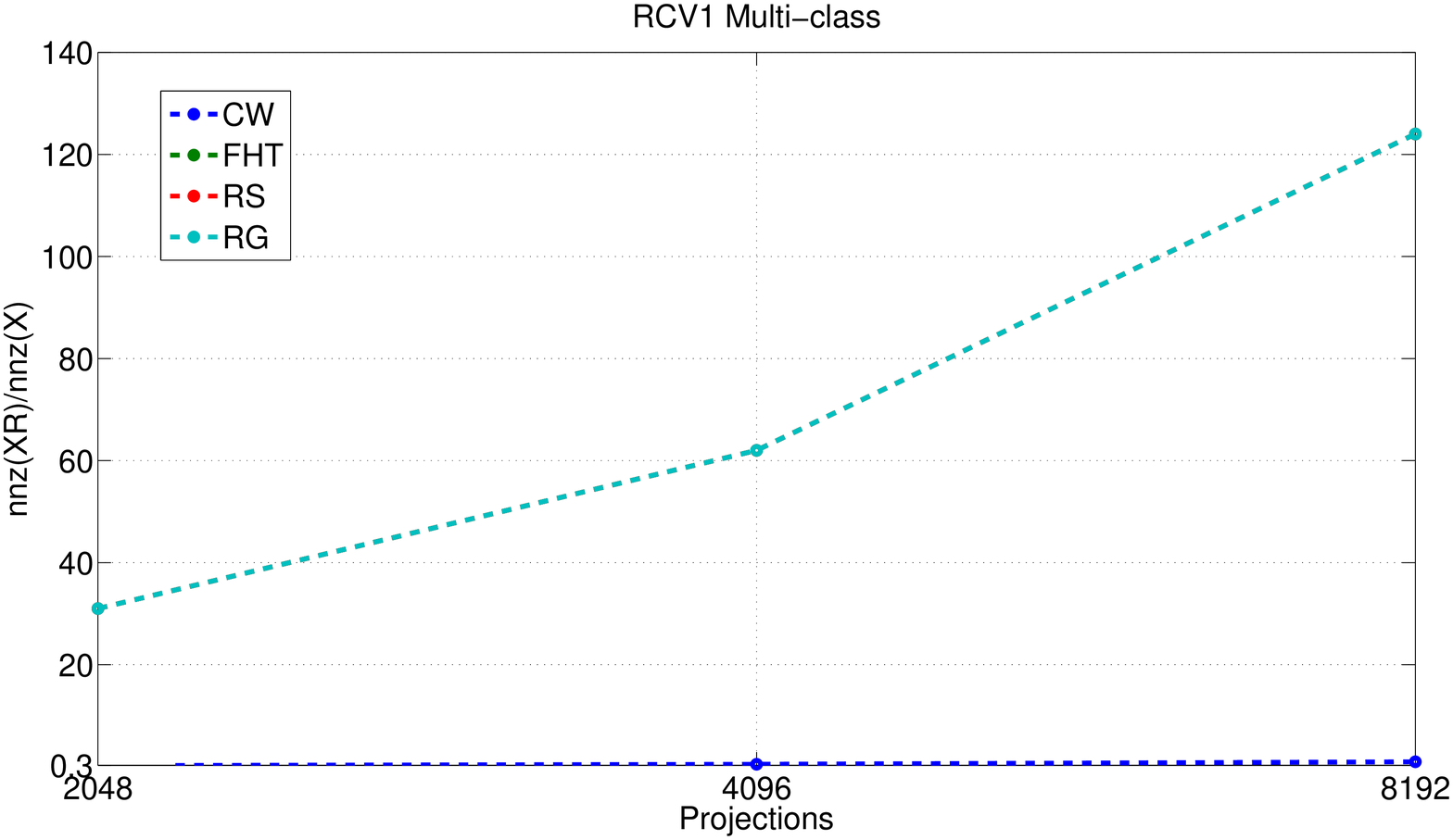}
\caption{\small{Ratio of number of non-zero entries of projected data and full-data for RCV1 dataset.}}
\label{fig:rcv1_nnz}
\end{figure}
%

\begin{table}[ht]
\tbl{RCV1 Dataset (Multi-class) \label{tab:rcv1_multiclass}}{

\begin{tabular}{|c||c|c|c|c|c|}
\hline
	\multicolumn{2}{|c|}{}&\multicolumn{3}{c|}{Projected Dimension \math{r}}&\\
	\cline{3-6}
	\multicolumn{2}{|c|}{}{}&2048 &4096 &8192 & \textbf{full} \\
	\hline
	\multirow{4}{*}{$\mathbf{\epsilon_{out}} $} &CW$(\mu)$ &18.30 &15.15  &13.50  &$\mathbf{11.83 }$  \\
							      &$(\sigma)$  &0.157  &0.117 &0.079  &$\mathbf{0.18}$\\
	\cline{2-6}
		                 &RS$(\mu)$  &17.57	&14.66  &13.21    &$\mathbf{11.83}$  \\
				&$(\sigma)$  &0.172	 &0.134  &0.1094  &$\mathbf{0.18}$\\
	\cline{2-6}				
		                 &FHT $(\mu)$  &17.50	&14.58 &13.2  &$\mathbf{11.83}$ \\
			& $(\sigma)$    &0.2033  &0.0905 &0.0597 &$\mathbf{0.18}$\\
	\cline{2-6}				
		                 &RG $(\mu)$ &17.47	&14.61  &13.21  &$\mathbf{11.83}$ \\
			& $(\sigma)$  &0.101	&0.146	 &0.036  &$\mathbf{0.18}$\\
	\hline\hline
	\multirow{4}{*}{$\mathbf{t_{rp} }$} &CW $(\mu)$  &0.0093	&0.0189	 &0.0357  &$\mathbf{--}$  \\
							& $(\sigma)$   &0.0005	&0.0007	&0.0010   &$\mathbf{--}$ \\
	\cline{2-6}
		                 &RS $(\mu)$   &2.412	&4.782	 &9.49  &$\mathbf{--}$ \\
				& $(\sigma)$     &0.0323	&0.0847	 &0.0636     &$\mathbf{--}$ \\		
	\cline{2-6}				
		                 &FHT $(\mu)$ 	&2.484	 &4.85 &9.966  &$\mathbf{--}$ \\
				&  $(\sigma)$   &0.1621	&0.023	 &0.23      &$\mathbf{--}$ \\
	\cline{2-6}				
		                 &RG $(\mu)$  &13.559	&28.875	&55.05	&$\mathbf{--}$ \\
				&  $(\sigma)$   &0.0715	&1.9801	 &1.4315  &$\mathbf{--}$ \\
	\hline\hline
	\multirow{4}{*}{$\mathbf{t_{svm} }$} &CW $(\mu)$ &1.99	&2.23	 &2.80	  &$\mathbf{2.96}$  \\
							& $(\sigma)$  &0.2309	&0.4030	 &0.2929   &$\mathbf{--}$ \\
	\cline{2-6}
		                 &RS $(\mu)$  &45.103	  &94.652  &207.278  &$\mathbf{2.96}$ \\
				&  $(\sigma )$  &2.82	&5.705	 &15.084       &$\mathbf{--}$ \\		
	\cline{2-6}				
		                 &FHT $(\mu)$  &45.468	 &98.287	&201.558 	&$\mathbf{2.96}$ \\
				&  $(\sigma)$  &3.143 &6.689	&10.765         &$\mathbf{--}$ \\
	\cline{2-6}				
		                 &RG $(\mu)$  	&47.458	&121.208	&263.392  	&$\mathbf{2.96}$ \\
				&  $(\sigma)$  &7.222	&23.393	&45.814        &$\mathbf{--}$ \\
	\hline
\end{tabular}}
\begin{tabnote}
\Note{RCV1 Multi-class: $\mu$ and $\sigma$ represents the mean and standard deviation of the results which have been averaged over ten different random projection matrices. Since there was one training set, there is no standard deviation for $t_{svm}$ of ``full" data.}
\end{tabnote}\
\end{table}	

\begin{table}[ht]
\tbl{RCV1 Dataset (binary-class) \label{tab:rcv1_binaryclass}}{

\begin{tabular}{|c||c|c|c|c|c|}
\hline
	\multicolumn{2}{|c|}{}&\multicolumn{3}{c|}{Projected Dimension \math{r}}&\\
	\cline{3-6}
	\multicolumn{2}{|c|}{}{}&2048 &4096 &8192 & \textbf{full} \\
	\hline
	\multirow{4}{*}{$\mathbf{\epsilon_{out}} $} &CW$(\mu)$  &12.90    &9.57    &6.54  &$\mathbf{4.51 }$  \\
							      &$(\sigma)$  &1.234   & 0.619    &0.190  &$\mathbf{0.1161}$\\
	\cline{2-6}
		                 &RS$(\mu)$  &9.26    &7.82   & 6.22   &$\mathbf{4.51}$  \\
				&$(\sigma)$  &0.134   & 0.116   & 0.077   &$\mathbf{0.1161}$\\
	\cline{2-6}				
		                 &FHT $(\mu)$  &9.23   & 7.94    &6.29    &$\mathbf{4.51}$ \\
			& $(\sigma)$   &0.166    &0.145    &0.150  &$\mathbf{0.1161}$\\
	\cline{2-6}				
		                 &RG $(\mu)$ &9.24    &7.75    &6.20  &$\mathbf{4.51}$ \\
			& $(\sigma)$   &0.255    &0.113    &0.109  &$\mathbf{0.1161}$\\
	\hline\hline
	\multirow{4}{*}{$\mathbf{\gamma }$} &CW $(\mu)$ &0.0123    &0.0094    &0.0125      &$\mathbf{0.0152}$  \\
							& $(\sigma)$   &0.0002	&0.0002    &0.0003  &$\mathbf{--}$ \\
	\cline{2-6}
		                 &RS $(\mu)$ &0.0158    &0.0105   & 0.0127   &$\mathbf{0.0152}$ \\
				& $(\sigma)$   &0.0005	&0.00009	&0.0001     &$\mathbf{--}$ \\		
	\cline{2-6}				
		                 &FHT $(\mu)$ 	&0.0160    &0.0105    &0.0127    &$\mathbf{0.0152}$ \\
				&  $(\sigma)$   &0.0003	 &0.0001	&0.0001      &$\mathbf{--}$ \\
	\cline{2-6}				
		                 &RG $(\mu)$   &0.0158    &0.0106    &0.0127   	&$\mathbf{ 0.0152}$ \\
				&  $(\sigma)$  &0.0004	&0.0001  &0.00009  &$\mathbf{--}$ \\
	\hline\hline
	\multirow{4}{*}{$\mathbf{t_{rp} }$} &CW $(\mu)$   &0.0112    &0.0231    &0.0456  &$\mathbf{--}$  \\
							& $(\sigma)$  &0.0008    &0.0014    &0.0008    &$\mathbf{--}$ \\
	\cline{2-6}
		                 &RS $(\mu)$   &4.34   &8.94   &18.89   &$\mathbf{--}$ \\
				& $(\sigma)$  &0.48    &1.33    &3.48      &$\mathbf{--}$ \\		
	\cline{2-6}				
		                 &FHT $(\mu)$ 	 &4.29   &8.21   &16.37  &$\mathbf{--}$ \\
				&  $(\sigma)$    &0.577    &1.003    &1.484       &$\mathbf{--}$ \\
	\cline{2-6}				
		                 &RG $(\mu)$  &20.52   &48.18  &86.384	&$\mathbf{--}$ \\
				&  $(\sigma)$   &0.248    &3.32   &6.15  &$\mathbf{--}$ \\
	\hline\hline
	\multirow{4}{*}{$\mathbf{t_{svm} }$} &CW $(\mu)$  &0.0726    &0.1512    &0.2909	  &$\mathbf{0.368}$  \\
							& $(\sigma)$  &0.0082    &0.0113    &0.009   &$\mathbf{--}$ \\
	\cline{2-6}
		                 &RS $(\mu)$  &9.22   &20.11   &41.29  &$\mathbf{0.368}$ \\
				&  $(\sigma )$   &0.9396    &3.245    &7.3966       &$\mathbf{--}$ \\		
	\cline{2-6}				
		                 &FHT $(\mu)$  &8.78   &18.95  &37.18	&$\mathbf{0.368}$ \\
				&  $(\sigma)$    &1.29    &2.79    &4.76      &$\mathbf{--}$ \\
	\cline{2-6}				
		                 &RG $(\mu)$  	&11.17  &30.17   &44.76   &$\mathbf{ 0.368}$ \\
				&  $(\sigma)$   &0.364   &4.204   &2.079       &$\mathbf{--}$ \\
	\hline
\end{tabular}}
\begin{tabnote}
\Note{RCV1 Binary-class: $\mu$ and $\sigma$ represents the mean and standard deviation of the results which have been averaged over ten different random projection matrices. Since there was one training set, there is no standard deviation for $t_{svm}$ and $\gamma$ of ``full" data.}
\end{tabnote}\
\end{table}

\subsection{PCA vs Random Projections}
Principal Components Analysis (PCA) constructs a small number of linear features that summarize the input data. PCA is computed by first mean-centering the features of the original data and then computing a low-rank approximation of the data matrix using SVD. Thus the PCA feature matrix is given by $\matZ = \matX_c\matV_k$, where $\matX_c$ represents the centered data matrix $\matX$ and $\matV_k$ represents the top $k$ right singular vectors of $\matX_c$. To the best of our knowledge, there is no known theoretical framework connecting PCA with margin or generalization error of SVM, so we only provide empirical evidence of the comparison.

Our goal is to evaluate if the data matrix represented by a small number of principal components can give the same or better performance than random projections when combined with SVMs, in terms of both running time and out-of-sample error. Note that the number of random projections is always greater than the rank of the matrix. For PCA we retain a number of principal components that is less than or equal to the rank of the matrix in order to compare its performance to random projections.

We used the TechTC300 dataset for experimental evaluation and used MATLAB's SVD solver in ``econ" mode to compute PCA. We kept $k$ equal to 32, 64, and $\rho$ (where $\rho$ is the rank of the matrix) principal components. The results corresponding to a number of principal components equal to the rank of the data matrices are referred to as ``full-rank'', while ``full'' refers to the results on the full-dimensional dataset. We ran PCA experiments on 294 datasets with $k=64$, since one of them had rank less than 64. For $k=32$, we used the entire set of 295 TechTC300 matrices. $t_{pca}$ denotes the time to compute PCA on the dataset and we set $C=500$ and $C=1$ in our experiments. The out-of-sample error for PCA is equal to or sometimes slightly better compared to the error on the full-dimensional datasets. Even though a smaller number of principal components achieve better out-of-sample error when compared to random projections, the combined running time of SVMs and PCA  is typically higher than that of random projections. The combined running time of SVMs and PCA is sensitive to the value of C, while the running time for random projections and SVM do not vary greatly, like PCA, by change of $C$.\footnote{ We repeated all experiments on TechTC-300 using $C=1$ and noticed the same pattern in the results as we did for $C=500$. RG and FHT are faster than the remaining two methods.}  Random projections are therefore a faster and more stable method than PCA. However, PCA appears to perform better than random projections when the number of components is equal to the rank of the matrix. Table \ref{tab:techtcpca} shows the results of PCA experiments; note that the standard deviation of $t_{run}$ for $C=500$ is quite high because of the varied running times of SVMs on the various TechTC300 matrices.

For a comparison of PCA to random projections, we consider the case of $r=512$ and the random gaussian matrix and randomized Hadamard Transform (see Table~\ref{tab:tabtechtc} for $C=500$), which has the best combined running time for random projections and SVMs. The combined running time of SVMs and ``full-rank'' PCA is smaller than that of RG (FHT) and SVMs by 0.19 (0.16) seconds, while the out-of-sample error of the former is only $4\%$ better. However, the time needed to compute PCA is 1.73 seconds, while random projections take negligible time to be computed; applying SVMs on the dimensionally-reduced matrix is the bottleneck of the computation. If PCA retains only 32 or 64 principal components, the running time of our method is smaller by factors of 30 and 5 respectively. 

For $C=1$ and $r=512$, SVMs and ``full-rank'' PCA is smaller than that of RG (FHT) and SVMs by 0.43 (0.15) seconds, while the out-of-sample error of the former is again $4\%$ better. For $C=1$, if PCA retains only 32 or 64 principal components, the running time of our method is smaller by a few seconds.

These clearly show the advantage of using random projections over PCA, especially when a small number of principal components is desired. The PCA feature matrix is sensitive to the value of $C$. For a higher value of $C$, SVM takes a longer time to train the inseparable data of the PCA feature matrix.

\begin{table}[!ht]
\tbl{TechTC300 PCA Experiments \label{tab:techtcpca}}{
\begin{tabular}{|c||c|c|c|c|c|}
\hline
	\multicolumn{2}{|c|}{C=500}&\multicolumn{3}{c|}{Projected Dimension \math{k}}&\\
	\cline{3-6}
	\multicolumn{2}{|c|}{}{}&32 &64 &\textbf{full-rank} & \textbf{full} \\
	\hline
	\multirow{1}{*}{$\mathbf{\epsilon_{out}} $} &PCA $(\mu)$ &15.02 &17.35 &17.35  &$\mathbf{17.35}$  \\
							      &$(\sigma)$ &9.32 &9.53 &9.45  &$\mathbf{9.45}$\\
	\hline\hline
	\multirow{1}{*}{$\mathbf{t_{pca} }$} &PCA $(\mu)$  &1.73 &1.73 &1.73 &$\mathbf{--}$  \\
							& $(\sigma)$    &1.15 &1.15 &1.15 &$\mathbf{--}$ \\
	
	\hline\hline
	\multirow{1}{*}{$\mathbf{t_{run} }$} &PCA $(\mu)$  &86.73  &14.04 &2.67 &$\mathbf{4.85}$  \\
							& $(\sigma)$   &174.35  &81.12 &1.56 &$\mathbf{2.12}$ \\
	\hline
\end{tabular}
\quad
\begin{tabular}{|c||c|c|c|c|c|}
\hline
	\multicolumn{2}{|c|}{C=1}&\multicolumn{3}{c|}{Projected Dimension \math{k}}&\\
	\cline{3-6}
	\multicolumn{2}{|c|}{}{}&32 &64 &\textbf{full-rank} & \textbf{full} \\
	\hline
	\multirow{1}{*}{$\mathbf{\epsilon_{out}} $} &PCA $(\mu)$  &13.33   &15.72   &17.21  &$\mathbf{17.21}$  \\
							      &$(\sigma)$  &8.10    &9.20    &9.44  &$\mathbf{9.44}$\\
	\hline\hline
	\multirow{1}{*}{$\mathbf{t_{pca} }$} &PCA $(\mu)$   &1.73 &1.73 &1.73  &$\mathbf{--}$  \\
							& $(\sigma)$    &1.15 &1.15 &1.15  &$\mathbf{--}$ \\
	
	\hline\hline
	\multirow{1}{*}{$\mathbf{t_{run} }$} &PCA $(\mu)$  &5.70    &2.99    &2.80   &$\mathbf{4.92}$  \\
							& $(\sigma)$  &10.68    &5.83    &1.62  &$\mathbf{2.16}$ \\
	\hline
\end{tabular}
}
\begin{tabnote}
\Note{PCA Experiments:}{Results on the TechTC300 dataset, averaged over all data matrices using PCA. The table shows how \math{\epsilon_{out}}, \math{t_{pca}} (in seconds), and \math{t_{run}} (in seconds) depend on $r$. $\mu$ and $\sigma$ indicate the mean and the standard deviation of each quantity over the data matrices and ten ten-fold cross-validation experiments.}
\end{tabnote}
\end{table}	

\subsection{Experiments on SVM regression}
We describe experimental evaluations on real world datasets, namely Yalefaces dataset \cite{CHHZ06} and a gene expression dataset (NCI60 \cite{nci60}). We convert the multi-label classification tasks into a regression problem. We use LIBSVM with default settings and use $C=1$ in all our experiments. The general observations from our experiments are as follows: (i) the combined runtime of random projections and SVM is smaller than the runtime of SVM on full dataset, (ii) the margin increases with an increase in the number of projections.

\subsubsection{Yalefaces Dataset}
\label{subsec:yale}
The Yalefaces dataset \cite{CHHZ06} consists of 165 grayscale images of 15 individuals. There were eleven images per subject, one per different facial expression (happy, sad, etc) or configuration (center-light, left-light, etc). The dataset has 165 datapoints and 4,096 features with 15 classes. The classes were used as the labels for regression. We set the value of $r$ to 256, 512, and 1024. $t_{run}$ for SVM and random projections is approximately 9, 7 and 4 times smaller than that of full-dataset. The margin increases as the number of random projections increases. The mean-squared in-sample error decreases with an increase in the number of random projections.

\begin{table}[!ht]

\tbl{Yalefaces Dataset \label{tab:yale}}{
\begin{tabular}{|c||c|c|c|c|c|}
\hline
	\multicolumn{2}{|c|}{}&\multicolumn{3}{c|}{Projected Dimension \math{r}}&\\
	\cline{3-6}
	\multicolumn{2}{|c|}{}{}&256 &512 &1024 & \textbf{full} \\
	\hline
	\multirow{4}{*}{$\mathbf{\epsilon_{in}} $} &CW$(mse)$ &0.2697    &0.0257    &0.0103   &$\mathbf{0.0098}$  \\
							      &$(\beta)$    &0.9859    &0.9987    &0.9995     &$\mathbf{0.9995}$\\
	\cline{2-6}
		                 &RS$(mse)$   &0.2201    &0.0187    &0.0107    &$\mathbf{0.0098 }$  \\
				&$(\beta)$   &0.9886    &0.9991   &0.9995       &$\mathbf{0.9995}$\\
	\cline{2-6}				
		                 &FHT $(mse)$   &0.2533    &0.0233    &0.0102    &$\mathbf{0.0098}$ \\
			& $(\beta)$    &0.9868    &0.9988    &0.9995     &$\mathbf{0.9995}$\\
	\cline{2-6}				
		            &RG $(mse)$   &0.281    &0.0174    &0.0105    &$\mathbf{0.0098}$ \\
			& $(\beta)$    &0.9853    &0.9991    &0.9995     &$\mathbf{0.9995}$\\

	\hline\hline
	\multirow{4}{*}{$\mathbf{\gamma}$ } &CW $(\mu)$   &0.11    &0.12    &0.13    &$\mathbf{0.14}$  \\
							&  $(\sigma)$  &0.0032    &0.0028    &0.0020  & $\mathbf{ 0.0031}$ \\
	\cline{2-6}
		                 &RS $(\mu)$  &0.11    &0.12    &0.13  &$\mathbf{0.14}$ \\
				& $(\sigma )$   &0.0043    &0.0024    &0.0018   &$\mathbf{ 0.0031}$ \\		
	\cline{2-6}				
		                 &FHT $(\mu)$  	&0.11    &0.12    &0.13       &$\mathbf{0.14}$ \\
				&  $(\sigma)$  &0.0026    &0.0022    &0.0031     &$\mathbf{ 0.0031}$ \\
	\cline{2-6}		 
				&RG $(\mu)$  	&0.11    &0.12    &0.13       &$\mathbf{0.14}$ \\
				&  $(\sigma)$  &0.0034    &0.0030    &0.0025     &$\mathbf{ 0.0031}$ \\
	\hline\hline
	\multirow{4}{*}{$\mathbf{t_{run} }$} &CW $(\mu)$   &4.14    &5.06    &8.62    &$\mathbf{34.93}$  \\
							& $(\sigma)$   &0.17    &0.12    &0.24   &$\mathbf{0.02}$ \\
	\cline{2-6}
			&RS $(\mu)$   &3.92    &4.81    &8.45   &$\mathbf{34.93}$ \\
				&  $(\sigma)$    &0.17    &0.10    &0.19       &$\mathbf{0.02}$ \\    	
	\cline{2-6}				
  			&FHT $(\mu)$  &4.12    &5.09    &8.69  &$\mathbf{34.93}$ \\
				&  $(\sigma )$    &0.23    &0.17    &0.25     &$\mathbf{0.02}$ \\
	\cline{2-6}				
  			&RG $(\mu)$  &4.37    &5.49    &9.43  &$\mathbf{34.93}$  \\
				&  $(\sigma )$    &0.15    &0.09    &0.06     &$\mathbf{0.02}$ \\
	\hline
\end{tabular}}
\begin{tabnote}
\Note{Yalefaces Dataset:}{Results on the Yalefaces dataset using four different random projection methods. The table shows how \math{\epsilon_{in}}, \math{\gamma} and \math{t_{run}} (in seconds) depend on $r$. $mse$ and $\beta$ indicate the mean-squared error and the squared correlation coefficient, while $\mu$ and $\sigma$ represent the mean and standard deviation over ten ten-fold cross-validation experiments, and ten choices of random projection matrices for the four methods that we investigated.}
\end{tabnote}
\end{table}

\subsubsection{NCI60 Dataset}
\label{subsec:nci}
The NCI60 dataset \cite{nci60} consists of 1375 gene expression profiles of 60 human cancer cell lines. The dataset contains 1375 features and 60 datapoints with ten classes. The features contain the log-ratio of the expression levels. The classes were used as labels for regression. We set the value of $r$ to 128, 256, and 512. The running time of the four methods are nearly the same. The squared correlation-coefficient is very close to one and is not influenced by the number of projections, $r$. The mean squared $\epsilon_{in}$ remains the same for all values of $r$.

\begin{table}[!ht]
\tbl{NCI60 Dataset \label{tab:nci60}}{
\begin{tabular}{|c||c|c|c|c|c|}
\hline
	\multicolumn{2}{|c|}{}&\multicolumn{3}{c|}{Projected Dimension \math{r}}&\\
	\cline{3-6}
	\multicolumn{2}{|c|}{}{}&128 &256 &512 & \textbf{full} \\
	\hline
	\multirow{4}{*}{$\mathbf{\epsilon_{in}} $} &CW$(mse)$  &0.0098    &0.0097    &0.0097  &$\mathbf{0.0097}$  \\
							      &$(\beta)$    &0.9987    &0.9989    &0.9989    &$\mathbf{0.9990}$\\
	\cline{2-6}
		                 &RS$(mse)$ &0.0098     &0.0097     &0.0097     &$\mathbf{0.0097}$  \\
				&$(\beta)$  &0.9987     &0.9988     &0.9990     &$\mathbf{0.9990}$\\
	\cline{2-6}				
		                 &FHT $(mse)$   &0.0097     &0.0098     &0.0097    &$\mathbf{0.0097}$ \\
			& $(\beta)$    &0.9987     &0.9988     &0.9989    &$\mathbf{0.9990}$\\
	\cline{2-6}				
		                 &RG $(mse)$   &0.0098     &0.0098     &0.0097    &$\mathbf{0.0097}$ \\
			& $(\beta)$    &0.9987     &0.9988     &0.9989    &$\mathbf{0.9990}$\\
	\hline\hline
	\multirow{4}{*}{$\mathbf{\gamma}$ } &CW $(\mu)$   &1.89    &2.12    &2.22     &$\mathbf{2.33}$  \\
							&  $(\sigma)$  &0.05    &0.07    &0.07  & $\mathbf{0.12}$ \\
	\cline{2-6}
		                 &RS $(\mu)$    &1.90    &2.09    &2.25   &$\mathbf{2.33 }$ \\
				& $(\sigma )$  &0.10    &0.06    &0.07    &$\mathbf{0.12}$ \\		
	\cline{2-6}				
		                 &FHT $(\mu)$  	 &1.88    &2.10    &2.22   &$\mathbf{2.33}$ \\
				&  $(\sigma)$    &0.10    &0.09    &0.05   &$\mathbf{0.12}$ \\
	\cline{2-6}				
		                 &RG $(\mu)$  	 &1.87    &2.12    &2.20   &$\mathbf{2.33}$ \\
				&  $(\sigma)$    &0.10    &0.10    &0.06   &$\mathbf{0.12}$ \\
	\hline\hline
	\multirow{4}{*}{$\mathbf{t_{run} }$} &CW $(\mu)$  &0.26    &0.39    &0.76    &$\mathbf{2.19}$  \\
							& $(\sigma)$  &0.01    &0.01    &0.01   &$\mathbf{0.003}$ \\
	\cline{2-6}
			&RS $(\mu)$  &0.27    &0.40    &0.78  &$\mathbf{2.19}$ \\
				&  $(\sigma)$    &0.02    &0.01    &0.01      &$\mathbf{0.003}$ \\    	
	\cline{2-6}				
  			&FHT $(\mu)$   &0.29    &0.43    &0.85   &$\mathbf{2.19}$ \\
				&  $(\sigma )$    &0.04    &0.05    &0.10     &$\mathbf{0.003}$ \\
	\cline{2-6}				
  			&RG $(\mu)$   &0.26    &0.38    &0.74   &$\mathbf{2.19}$ \\
				&  $(\sigma )$    &0.02    &0.04    &0.05     &$\mathbf{0.003}$ \\
	
	\hline
\end{tabular}}
\begin{tabnote}
\Note{NCI60 Dataset:}{Results on the NCI60 dataset using four different random projection methods. The table shows how \math{\epsilon_{in}}, \math{\gamma} and \math{t_{run}} (in seconds) depend on $r$. See caption of Table ~\ref{tab:yale} for an explanation of $mse$, $\beta$, $\mu$ and $\sigma$.}
\end{tabnote}
\end{table}	

\section{Conclusions and open problems}
We present theoretical and empirical results indicating that random projections are a useful dimensionality reduction technique for SVM classification and regression problems that handle sparse or dense data in high-dimensional feature spaces.
Our theory predicts that the dimensionality of the projected space (denoted by $r$) has to grow essentially \textit{linearly} (up to logarithmic factors) in $\rho$ (the rank of the data matrix) in order to achieve relative error approximations to the margin and the radius of the minimum ball enclosing the data. Such relative-error approximations imply excellent generalization performance. However, our experiments show that considerably smaller values for $r$ results in classification that is essentially as accurate as running SVMs on all available features, despite the fact that the matrices have full numerical rank. This seems to imply that our theoretical results can be improved. 
We implemented and tested random projection methods that work well on dense matrices (the RS and FHT methods of Section~\ref{subsec:prel}), as well as a very recent random projection method that works well with sparse matrices (the CW method of Section~\ref{subsec:prel}). We also experimented with different SVM solvers for lage and medium-scale datasets. As expected, FHT, RG and RS work well on dense data while CW is an excellent choice for sparse data, as indicated by the SVM classification experiments.
For large-scale sparse data, CW is the method of choice as the other methods outweigh the benefits of performing random projections. 
For SVM regression experiments, the combined running times using the four methods are the same for dense datasets. The mean squared error and the squared correlation-coefficient of $\epsilon_{in}$ of the projected data are non-zero as opposed to the SVM classification experiments. Finally, we compare random projections with a popular method of dimensionality reduction, namely PCA and see that the combined running time of random projection and SVM is faster than that of SVM and PCA, with a slightly worse out-of-sample error. 
All our experiments are on matrices of approximately low-rank, while the theory holds for matrices of exactly low-rank. It is not known if the theory extends to matrices of approximately low rank. This is an open problem and needs further investigation.

\bibliographystyle{ACM-Reference-Format-Journals}
\bibliography{references}

\received{April 2013}{October 2013}{December 2013}

\end{document}